\documentclass[preprint,3p,authoryear]{elsarticle} % 1 colum % 리뷰단계는 1colm 권장한다고 저널홈페이지에 명시됨
\usepackage{amssymb}
\usepackage{amsmath}
\usepackage{amsthm}
\newtheorem*{prop}{Proposition}
\newtheorem*{remark1}{Remark 1}
\newtheorem*{remark2}{Remark 2}
\usepackage{kotex}
\usepackage[pdftex]{color}
\usepackage[font=footnotesize,labelfont=bf]{caption}
\usepackage{subcaption}
\usepackage{caption}
\usepackage{multirow}
\usepackage{floatrow}
\usepackage{float}
\usepackage{algorithm}
\usepackage{algpseudocode}
\usepackage{xfrac}

\usepackage{graphicx}% Include figure files
\usepackage{dcolumn}% Align table columns on decimal point
\usepackage{bm}% bold math
\usepackage{tabularx}
\usepackage{hyperref}
\usepackage{enumitem}
\usepackage{nicematrix}
\usepackage{tikz}
\usetikzlibrary{tikzmark,calc}
\newsavebox{\imagebox}
\usepackage{blindtext}
\usepackage{lineno}
\journal{Engineering Applications of Artificial Intelligence}

\begin{document}

\begin{frontmatter}

\title{Towards Reliable Uncertainty Quantification via Deep Ensemble\\
in Multi-output Regression Task}
% \title{Towards Quantifying Calibrated Uncertainty via Deep Ensembles in Missile Performance Regression Tasks}

\author[1]{Sunwoong Yang}
\ead{sunwoongy@kaist.ac.kr}

\author[2]{Kwanjung Yee\corref{cor1}}
\ead{kjyee@snu.ac.kr}

\cortext[cor1]{Corresponding author}

\address[1]{Cho Chun Shik Graduate School of Mobility, Korea Advanced Institute of Science and Technology, Daejeon, 34051, Republic of Korea}

\address[2]{Department of Aerospace Engineering, Seoul National University, Seoul, 08826, Republic of Korea}

% \blfootnote{This work was presented at the AIAA SciTech 2023 Forum, 23-27 January 2023, National Harbor, MD & Online \citep{yang2023uncertainty}} \blindtext
	
\let\thefootnote\relax\footnote{This work was presented at the AIAA SciTech 2023 Forum, 23-27 January 2023, National Harbor, MD \& Online \citep{yang2023uncertainty}}

\begin{abstract} 
% \begin{linenumbers}

This study aims to comprehensively investigate the deep ensemble approach, an approximate Bayesian inference, in the multi-output regression task for predicting the aerodynamic performance of a missile configuration. To this end, the effect of the number of neural networks used in the ensemble, which has been blindly adopted in previous studies, is scrutinized. As a result, an obvious trend towards underestimation of uncertainty as it increases is observed for the first time, and in this context, we propose the deep ensemble framework that applies the post-hoc calibration method to improve its uncertainty quantification performance. It is compared with Gaussian process regression and is shown to have superior performance in terms of regression accuracy ($\uparrow55\sim56\%$), reliability of estimated uncertainty ($\uparrow38\sim77\%$), and training efficiency ($\uparrow78\%$). Finally, the potential impact of the suggested framework on the Bayesian optimization is briefly examined, indicating that deep ensemble without calibration may lead to unintended exploratory behavior. This UQ framework can be seamlessly applied and extended to any regression task, as no special assumptions have been made for the specific problem used in this study.

% \end{linenumbers}
\end{abstract}


\end{frontmatter}


\section{Introduction}
\label{sec:intro}

We are entering an era of high-performance computing technologies and they have enabled engineers to efficiently obtain vast amounts of data, so-called big data. Accordingly, numerous data-driven approaches have been studied to derive physical insights from the growing number of available datasets. The most popular but most fundamental one is to utilize a given dataset to train a regression model (also referred to as a surrogate model), which is used to predict quantities of interest (QoIs) \citep{jeong2005efficient, nikolopoulos2022non, yang2022inverse, hong2023exploration}. This straightforward approach can be leveraged for a variety of applications, from exploration during the design optimization process \citep{yang2022design} to the prediction of high-dimensional data via reduced-order modeling \citep{kang2022pof}. Furthermore, from the perspective that the regression model can accelerate the realization of digital twins by replacing the high-demand simulations required within its procedure \citep{vanderhorn2021digital}, its potential seems boundless.

However, such impacts cannot be fully achieved by the regression model alone. In real-world engineering problems, \textit{knowing what it does not know and therefore improving interpretability} is an indispensable issue. In the decision-making process based on the regression model, engineers should consider the predictive uncertainty derived from insufficient train data and imperfect regression model \citep{zhang2022uncertainty}. Otherwise, blind faith in regression models, especially during risk assessment and management procedures, can lead to unexpected and therefore disastrous outcomes. The most common approach to deal with this issue is to perform Bayesian optimization, also known as efficient global optimization in engineering fields \citep{jones1998efficient, yang2022design, chae2010helicopter, kanazaki2007multi}. Briefly, it aims to reduce model uncertainty by iteratively updating the model based on the acquisition function \citep{snoek2012practical, shin2020deep, shimoyama2013updating}, which contains uncertainty information (Fig. \ref{fig:EGO}). Since the Bayesian optimization process requires uncertainty quantification (UQ), whether the model quantifies the uncertainty over its prediction is the key consideration for engineers in determining which regression model to utilize.

\begin{figure*}[htb!]
    \centering

        \includegraphics[width=.9\textwidth]{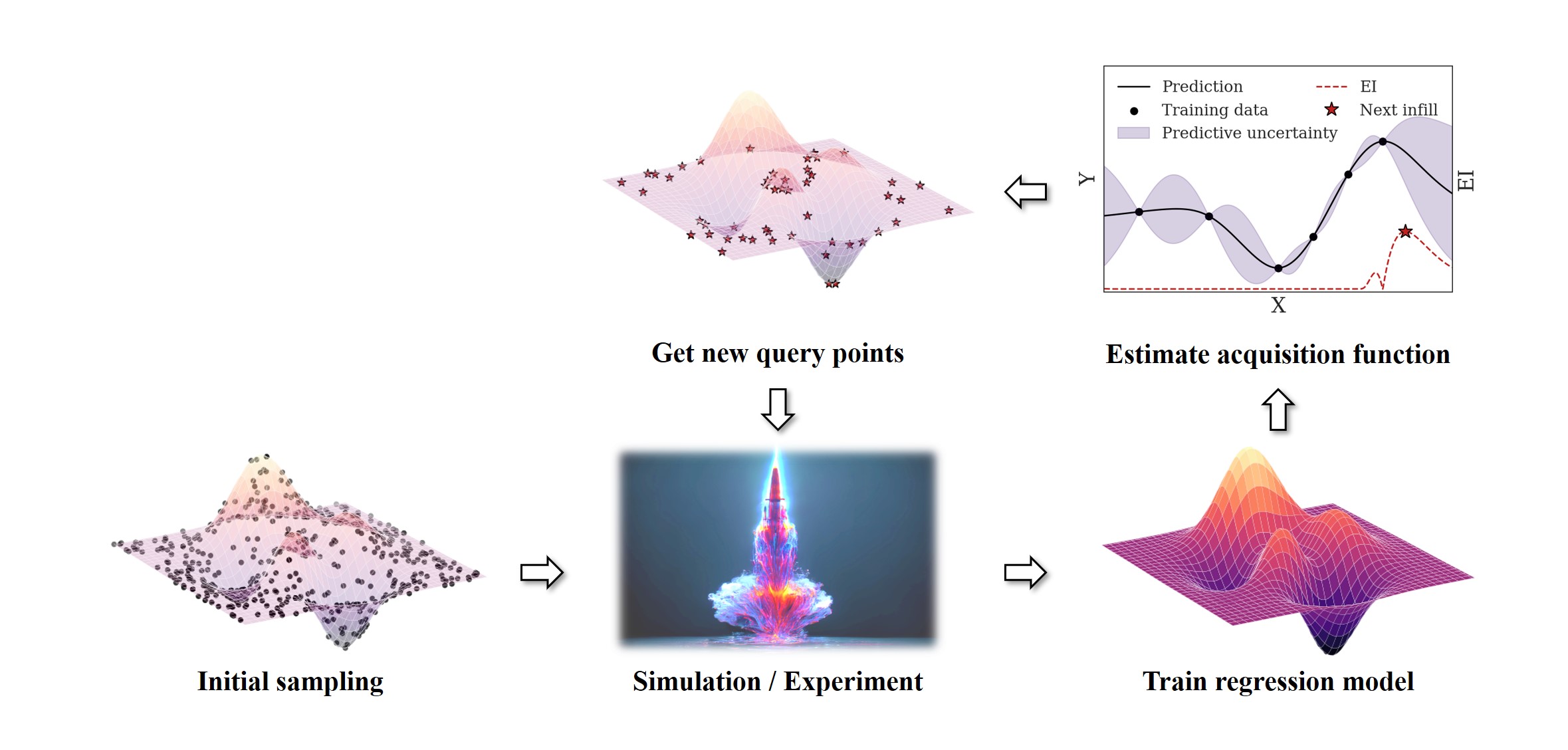}
        
    \caption{Flowchart of Bayesian optimization.}
    \label{fig:EGO}
\end{figure*} 

Gaussian process regression (GPR)---also known as Kriging---is one of the most widely used regression models capable of UQ in various engineering fields \citep{rhode2020non, wang2022dynamic, yang2020surrogate, quirante2018hybrid, quirante2016optimization, kessler2019global, zhong2019operation, yang2022comment, sugimura2009kriging, park2022multi, munoz2023gaussian, yildiz2020slime, yildiz2022marine}. GPR allows engineers to identify which predictions are unreliable by providing predictive uncertainty, and it has become the most prevalent regression model for Bayesian optimization \citep{jeong2005efficient, yang2022design, namura2016efficient, shimoyama2013updating}. However, GPR is notorious for its time complexity of $O(n^3)$ and memory complexity of $O(n^2)$, where $n$ denotes the dataset size \citep{cheng2017variational, wang2017time}. Even in multi-output regression tasks, since a GPR is trained for each output independently, the required training time increases linearly with respect to the output dimension, and the correlations within outputs become completely ignored \citep{lin2021multi, lin2022gradient, wang2015gaussian}.  

In this regard, Bayesian neural networks (BNNs) \citep{mackay1992information, mackay1995probable, fernandez2023physics} can be effective alternatives for the following reasons: 1) their universal approximation capability \citep{hornik1989multilayer, barron1993universal}; 2) scalability to large datasets due to mini-batch training \citep{meng2021multi}; and 3) multi-output prediction only with a single regression model. Since BNNs aim to learn the probability distributions of the model parameters on the basis of Bayesian inference, they can estimate the uncertainty of their prediction, whereas traditional neural networks (NNs) only provide point estimates. However, their additional model parameters lead to slower convergence during the training \citep{zhang2019quantifying} and require significant modifications to the conventional framework of NNs, leading to cumbersome and knotty training algorithms \citep{gal2016uncertainty, lakshminarayanan2017simple, fernandez2022uncertainty}. Such computational complexity and inefficiency prevent BNNs from being a viable option for engineers who prioritize practicality and are not familiar with Bayesian formalism.

Recently, easy-to-use but scalable approaches for approximating Bayesian inference have attracted the attention of engineers. Especially, deep ensemble (DE) \citep{lakshminarayanan2017simple} and MC-dropout \citep{gal2016dropout, deruyttere2021giving} require only a few modifications to standard (or vanilla) NNs, demonstrating their applicability to the fields of engineering. However, since MC-dropout has controversial issues about whether or not it is Bayesian inference \citep{osband2016risk, hron2017variational, hron2018variational, folgoc2021mc}, it is out of our focus; see \ref{sec:app_MCD}. DE, an approach to quantify the predictive uncertainty by leveraging ensembles of NNs, was first proposed by \citet{lakshminarayanan2017simple}. Their idea is so ``simple and straightforward'' that it only requires training multiple NNs in parallel on the same training dataset. Despite its simplicity, several researchers have recognized that the DE provides not only accurate predictions, but also robust, reliable, and practically useful uncertainty on a wide variety of architectures and datasets, even on out-of-distribution (OOD) examples \citep{gustafsson2020evaluating,fort2019deep,ovadia2019can,ashukha2020pitfalls}. Finally, it has come to be treated as the ``gold standard for accurate and well-calibrated predictive distributions'' \citep{wilson2020bayesian}.

However, most previous studies have focused on verifying whether DE accurately estimates the uncertainty in classification tasks \citep{lakshminarayanan2017simple, rahaman2021uncertainty, wu2021should, fort2019deep, ovadia2019can, ashukha2020pitfalls}. Its comprehensive validation has not been conducted in multi-output regression tasks, which are the most common problems in practical engineering disciplines. For example, \citet{de2021bayesian} and \citet{pawar2022multi} utilized DE for tailless aircraft range optimization and boundary layer flow prediction tasks, respectively, without any validation of the estimated uncertainty in their problems. In this sense, our research focuses on a thorough validation of the DE approach in multi-output regression tasks, while comparing it with GPR, both in terms of regression accuracy and reliability of the estimated uncertainty. Especially, we seek to overcome the limitations of existing studies that blindly adopted the number of NNs used in DE without sufficient explanation of their effects \citep{pocevivciute2022generalisation, ovadia2019can, ilg2018uncertainty, de2021bayesian, linmans2020efficient, rahaman2021uncertainty, de2021bayesian, egele2022autodeuq, maulik2023quantifying}. Finally, a tendency of the quantified uncertainty to become underconfident with the number of NNs is observed and a practical calibration method is proposed to be applied. The corresponding effects are verified quantitatively with two uncertainty evaluation criteria, and their potential impact on Bayesian optimization is briefly investigated. The main contributions of this work can be summarized as follows:
\begin{enumerate}
    
    \item First attempt to validate DE approach in the multi-output regression task.
    
    \item The effect of the number of NNs used for DE is comprehensively investigated and two different criteria are utilized for rigorous validation of its uncertainty quality.
 
    \item Accordingly, an increasing trend of underconfidence with the increasing number of NNs is first empirically observed in the regression task, and its analytical explanation is derived.
    
    \item A simple post-hoc calibration method is applied to DE models for the correction of unsatisfactory uncertainty quality and its effectiveness is verified both qualitatively and quantitatively.
    
    \item The potential impact of the proposed calibration method on Bayesian optimization is briefly examined: the possibility that different estimates of uncertainty could lead to different exploration behavior is examined.

    \item Throughout the above procedures, GPR---the most well-known UQ model---is compared with DE, and the effectiveness of DE over GPR is confirmed.
    
\end{enumerate}

The rest of this paper is organized as follows. In Section \ref{sec:DE}, the background on how to implement DE and evaluate its uncertainty quality is described. In Section \ref{sec:DE_missile}, the application of DE to a multi-output regression task in aerospace engineering is elaborated. It provides a thorough validation of DE models compared to GPR models, both in terms of prediction accuracy and uncertainty quality. In Section \ref{sec:calib_results}, a simple post-hoc calibration method is applied and its effects on uncertainty quality and Bayesian optimization are investigated. Finally, in Section \ref{sec:conclusion}, the conclusion and future work of this study are presented.

\clearpage
\section{Implementation and evaluation of DE}
\label{sec:DE}

DE was first proposed by \citet{lakshminarayanan2017simple} for the simple and scalable estimation of predictive uncertainty. Although its idea can be seen as a straightforward extension of NNs (making use of multiple NNs), DE has received little attention in the engineering disciplines, in contrast to its reputation in computer science. This is due to the lack of previous works explaining its algorithm friendly and comprehensively, and therefore the purpose of this section is to fill the academic gap by elaborating on the DE methodology and its validation. First, we briefly introduce the NNs (Section \ref{sec:DE_NN}) before moving on to DE. Then, the background of how to implement DE (Section \ref{sec:DE_DE}) and how to evaluate its uncertainty quality (Section \ref{sec:DE_UQ}) is described.

\subsection{\label{sec:DE_NN} Neural networks (NNs)}

Engineers from various disciplines have been drawn to NNs due to their universal approximation capability \citep{hornik1989multilayer, barron1993universal}, ability to scale to large datasets through mini-batch training \citep{meng2021multi}, and capability of making multi-output predictions with a single regression model. This section provides a brief theoretical overview of these NNs.

The feed-forward mechanism propagates the data obtained from the input layer of NNs to the output layer. In this procedure, information moves via an affine transformation as follows:
\begin{equation}\
\label{eq:affine}
y = Wx + b,
\end{equation}
where $x$ is a vector of nodes in the input layer and $y$ is that in the output layer. $W$ and $b$ are the weight matrix and bias vector between the input and output layers, respectively. Regardless of the number of hidden layers between the input and output layers, nonlinearity between $x$ and $y$ cannot be captured since they are linearly correlated in Eq. \ref{eq:affine}. In this context, the concept of an activation function that modifies the output of NNs is introduced. By incorporating nonlinear activation functions at each layer, NNs can perform nonlinear modeling. A variety of activation functions are available, including the LeakyReLU function \citep{maas2013rectifier}, which is as follows:
\begin{equation}\
\label{eq:act}
    f(x)= 
\begin{cases}
    x,& \text{if } x \ge 0\\
    ax,              & \text{otherwise}
\end{cases}
\end{equation}
where $a$ stands for a non-zero small gradient (0.01 for this study). And this activation function (Eq. \ref{eq:act}) is applied to the output of the previous layer. The correspondingly transformed output is then utilized as the input for the subsequent layer. This process, known as feed-forward, is repeated through the hidden layers.

However, the feed-forward itself cannot achieve the expected accuracy because it lacks an algorithm for training the parameters of NNs, namely weights ($W$) and biases ($b$). To address this issue, the backpropagation training algorithm was introduced, which minimizes the loss function by adjusting the parameters to make the predicted values of the NNs similar to the desired target values as the training progresses \citep{rumelhart1986learning}. To this end, gradient descent optimization techniques, such as Adagrad \citep{duchi2011adaptive}, RMSprop \citep{tieleman2012lecture}, and Adam \citep{kingma2014adam}, are utilized to minimize the loss function. In particular, Adam has become increasingly popular due to its strengths in dealing with sparse gradients and non-stationary objectives, combining Adagrad and RMSprop \citep{kingma2014adam}. As the feed-forward process and backpropagation with gradient descent are repeated iteratively, the loss function decreases to the desired level so that the training stops. The converged weights and biases of the NN model can then be used to make almost real-time predictions using the feed-forward operation. Only the essential aspects of NNs are presented here, as many studies have already described them. More information on NNs can be found in \citet{goodfellow2016deep}.

\clearpage
\subsection{\label{sec:DE_DE} Deep ensemble (DE)}

The NNs discussed above are often considered ``overconfident'' because they do not provide any measure of uncertainty. For those who are interested in UQ, DE can be an alternative approach. DE is based on an ensemble of NNs, but there is a key distinction: unlike a standard NN, which only outputs QoIs as $\mu(x)$, the NN used for DE outputs them as a Gaussian distribution, $N\bigl(\mu(x),\sigma^2(x)\bigr)$. That is, it assumes that QoIs are sampled from $N\bigl(\mu(x),\sigma^2(x)\bigr)$ and aims to provide information about this distribution by outputting $\mu(x)$ and $\sigma^2(x)$. Here, $\mu(x)$ refers to the estimated/predicted value and $\sigma^2(x)$ refers to the estimated/predicted variance. It should be noted that the estimated variance $\sigma^2(x)$ indicates the aleatory uncertainty (uncertainty arising from noise inherent in the training data) regarding the estimated value $\mu(x)$ \citep{solopchuk2021active, laves2021recalibration}. With this specific NN architecture, the number of final nodes is doubled since it outputs not only the standard outputs, $\mu(x)$, but also the uncertainty about them, $\sigma^2(x)$. Due to the probabilistic distribution it provides, this type of NN is referred to as a probabilistic NN.

The probabilistic NN architecture is adopted in the DE model since the vanilla NN structure cannot apply the proper scoring rule, which is the criterion for estimating the quality of predictive uncertainty \citep{gneiting2007strictly}. \citet{lakshminarayanan2017simple} emphasized that with the vanilla NN architecture, which provides only the estimated value $\mu(x)$, the mean squared error (MSE) would be used as the loss function:
\begin{equation}\
\label{eq:MSE}
\mathrm{MSE}=\bigl(y-\mu(x)\bigr)^2
\end{equation}
and therefore the information about the predictive uncertainty is entirely disregarded during the training. To address this issue, they proposed utilizing a probabilistic NN that can output both $\mu(x)$ and $\sigma^2(x)$. It allows the use of the proper scoring rule, negative log-likelihood (NLL), which is the standard metric for assessing the quality of probabilistic models \citep{hastie2009elements}:

\begin{equation}\
\label{eq:NLL}
\mathrm{NLL(\mu(x), \sigma^2(x), y)}=-\mathrm{log}\bigl(p_{\theta}(y|x)\bigr)=\cfrac{\mathrm{log}{\sigma^2(x)}}{2}+\cfrac{\bigl(y-\mu(x)\bigr)^2}{2\sigma^2(x)}+\cfrac{\mathrm{log}2\pi}{2}
\end{equation}
This NLL allows the intuitive interpretations as follows \citep{kendall2017uncertainties, guo2017calibration}. 1) When some training points have high MSE, $\bigl(y-\mu(x)\bigr)^2$, the impact of the term $\tfrac{\bigl(y-\mu(x)\bigr)^2}{2\sigma^2(x)}$ is relatively significant compared to $\tfrac{\mathrm{log}{\sigma^2(x)}}{2}$. Therefore, the model is trained to output high denominator value, $\sigma^2(x)$, at the corresponding points to reduce the NLL. 2) At training points with low MSE, the term $\tfrac{\mathrm{log}{\sigma^2(x)}}{2}$ becomes relatively dominant and thus the model is encouraged to output low $\sigma^2(x)$ at those points. In summary, the NLL scoring rule-based training algorithm for the probabilistic NN facilitates the learning of reliable predictive uncertainty by estimating high uncertainty where prediction error is high and low uncertainty where prediction error is low. It should be noted that this cannot be accomplished in vanilla NN with MSE loss function.

However, using a single probabilistic NN is limited to estimating the aleatory uncertainty. To estimate the epistemic uncertainty arising from the model parameters due to insufficient training data, a further step is required. \citet{lakshminarayanan2017simple} suggested the use of multiple probabilistic NNs, called deep ensemble (DE), to quantify both aleatory and epistemic uncertainties. Specifically, they aimed to capture the epistemic uncertainty by using the multiple probabilistic NNs trained on the identical dataset (also identical architectures for NNs are used). The overall training procedure is summarized in Algorithm \ref{alg:train}. There are two notable points herein: 1) the random initialization of the model parameters of the NNs in line 2; and 2) the random shuffling of the training dataset due to mini-batches in line 5. These two factors are regarded as the main causes of the individual NN with identical architecture in the ensemble being able to be trained with enough diversity \citep{lakshminarayanan2017simple}. See \citet{fort2019deep} for further information, which examined the effects of random initialization and random shuffling.

\begin{algorithm}[htb!]
\caption{Training procedure of DE}\label{alg:train}
    \begin{algorithmic}[1]
        \State Split the train dataset $X$ (with input $x$ and output $y$) into $J$ mini-batches.
        \State Randomly initializes model parameters of the $M$ probabilistic NNs and set training epochs.
        \For{$i=1:M$} \Comment{Loop for NN (parallelizable)}
            \For{epochs} \label{epochs} \Comment{Loop for epoch}
                \For{$j=1:J$} \Comment{Loop for mini-batch}
                    \State{$\mu_{ij}, \sigma^2_{ij}=$ NN$_{i}(x_{ij})$} \Comment{Feed-forward with mini-batch $x_{j}$}
                    \State{$\mathcal{L}_{ij}=$NLL$(\mu_{ij}, \sigma^2_{ij}, y_{ij})$} \Comment{Calculate NLL}
                    \State{$\theta_{i}=\theta_{i}-{learning\,rate}*{\delta{\mathcal{L}_{ij}}}/{\delta{\theta}}$} \Comment{Update model parameters of NN$_{i}$}
                \EndFor 
            \EndFor
        \EndFor
    \end{algorithmic}
\end{algorithm}

To see how the ensemble of probabilistic NNs trained in Algorithm \ref{alg:train} estimates two types of uncertainty, let $\mu_{i}(x)$ and ${\sigma}^2_{i}(x)$ be the predictive mean and predictive variance output by the $i$th individual NN. Herein, the predicted probabilities of $y$ from the $i$th NN can be expressed as $N\bigl(\mu_{i}(x),\sigma^2_{i}(x)\bigr)$, indicating that there are multiple Gaussian distributions according to each NN in the ensemble. \citet{lakshminarayanan2017simple} suggested approximating the final probability of the output as a mixture of Gaussian probabilities as follows:
\begin{equation}\
\label{eq:mixture_m}
\hat{\mu}=\cfrac{1}{M}\sum\limits_{i=1}^M{\mu_{i}},
\end{equation}

\begin{equation}\
\label{eq:mixture_s}
\begin{aligned}
\underbrace{\hat{\sigma}^2}_{\substack{\text{predictive}\\ \text{uncertainty}}}
& =\cfrac{1}{M}\sum\limits_{i=1}^M{\sigma^2_{i}}
+ (\cfrac{1}{M}\sum\limits_{i=1}^M{\mu^2_{i}}-\hat{\mu}^2) \\
& = \underbrace{E(\sigma_i)}_{\substack{\text{aleatory}\\
\text{uncertainty}}} + \underbrace{Var(\mu_i)}_{\substack{\text{epistemic}\\ \text{uncertainty}}} \\
\end{aligned}
\end{equation}
where $M$ is the number of probabilistic NNs used for the ensemble. Accordingly, the final predictive value of DE is $\hat{\mu}$ and the final predictive uncertainty is $\hat{\sigma}^2$. As in Eq. \ref{eq:mixture_s}, the predictive uncertainty can be decomposed into aleatory and epistemic uncertainty; see \citet{scalia2020evaluating} and \citet{hu2021learning} for more details. It should be noted that no additional training algorithm is required after the training of probabilistic NNs in Algorithm \ref{alg:train}: only the mixture process of already trained NNs in Eq. \ref{eq:mixture_m} and Eq. \ref{eq:mixture_s} is required. The overall flowchart of DE from the training of probabilistic NNs to the final prediction is schematically shown in Fig. \ref{fig:DE_struct}.

\begin{figure*}[htb!]
    \centering
    
        \includegraphics[width=.7\textwidth]{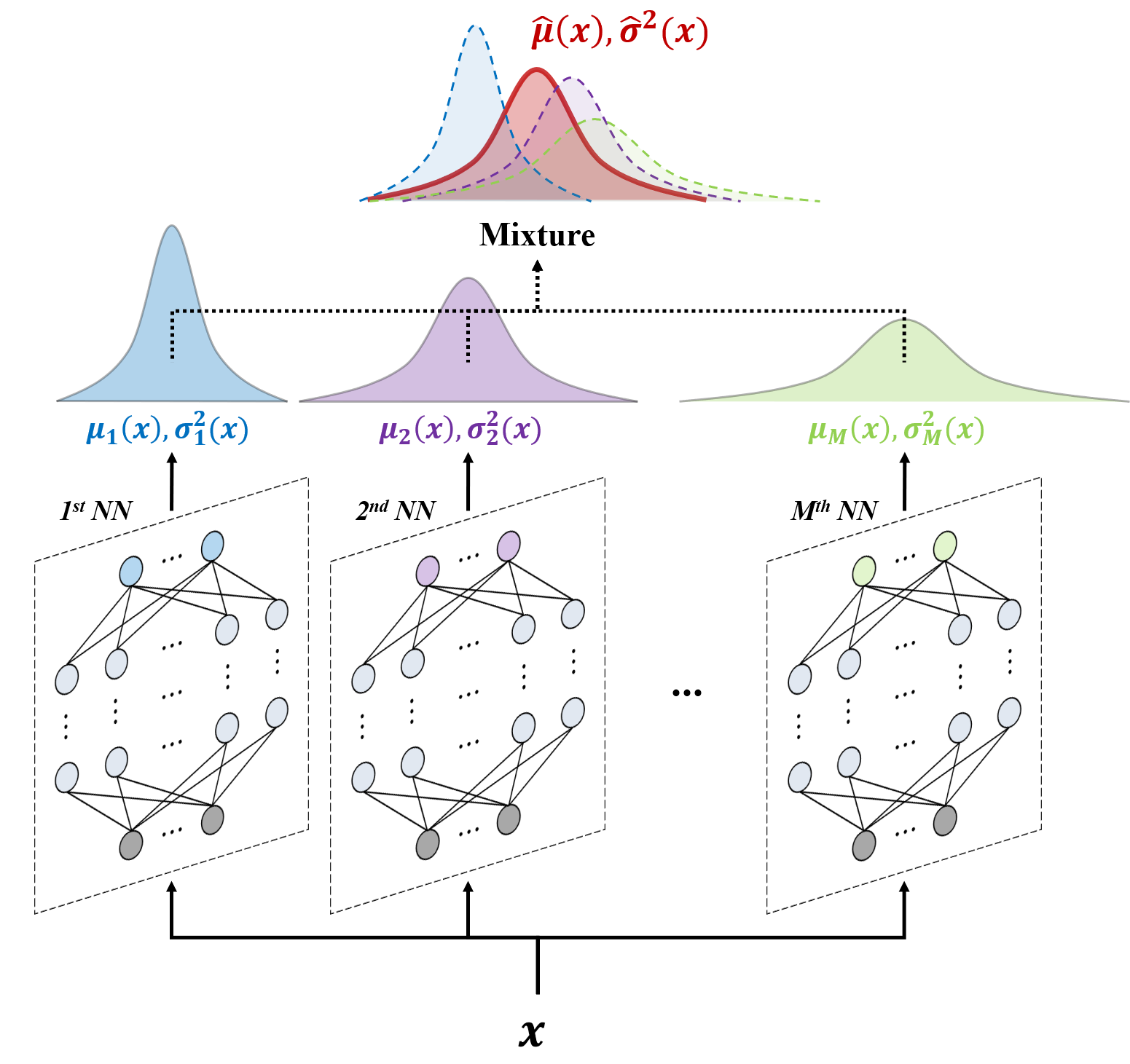}
        
    \caption{Flowchart of DE approach.}
    \label{fig:DE_struct}
\end{figure*} 

\subsection{\label{sec:DE_UQ} Uncertainty quality evaluation}

In the previous Section \ref{sec:DE_DE}, we explored the ability of the DE technique to determine predictive uncertainty. However, engineers who are interested in predicting uncertainty require more than just the feasibility of UQ; they also require confidence in the reliability of the estimated uncertainty. Unfortunately, previous studies that employed GPR to evaluate predictive uncertainty in engineering fields have disregarded this point. Consequently, the purpose of this section is to address this gap by presenting two criteria for assessing the accuracy/reliability of estimated predictive uncertainty. These techniques are applicable to any regression model performing UQ, such as GPR and DE.

\subsubsection{\label{sec:AUCE} AUCE}

The most widely used metric to evaluate the reliability of uncertainty is the area under the calibration error curve (AUCE) \citep{kuleshov2018accurate, gustafsson2020evaluating}. The primary goal of this measure is to ensure that the confidence intervals (CI) estimated by the model are accurate in practice. The concept of AUCE is shown schematically in Fig. \ref{fig:AUCE_info}. In Fig. \ref{fig:AUCEa}, the CI labeled ``Well-calibrated 60\% CI'' contains 60\% of the test dataset (6 out of 10 points), where test dataset indicates the dataset used to verify the quality of the estimated uncertainty. Thus, a well-calibrated model would have a 60\% CI that actually contains 60\% of the test data. On the other hand, if the 60\% CI contains more than 60\% of the dataset (8 out of 10 points), the model is considered underconfident, which corresponds to the case of ``Underconfident 60\% CI.'' This means that the model is not confident enough about its prediction and overestimates its CI. Conversely, if the 60\% CI contains less than 60\% of the dataset (4 out of 10 points, ``Overconfident 60\% CI'' case), the model is considered overconfident, meaning that it is too confident in its prediction and thus estimates a narrower CI than it actually should.

\begin{figure*}[htb!]
    \centering
    \begin{subfigure}[t]{0.4\textwidth}
        \centering
        \includegraphics[width=\linewidth]{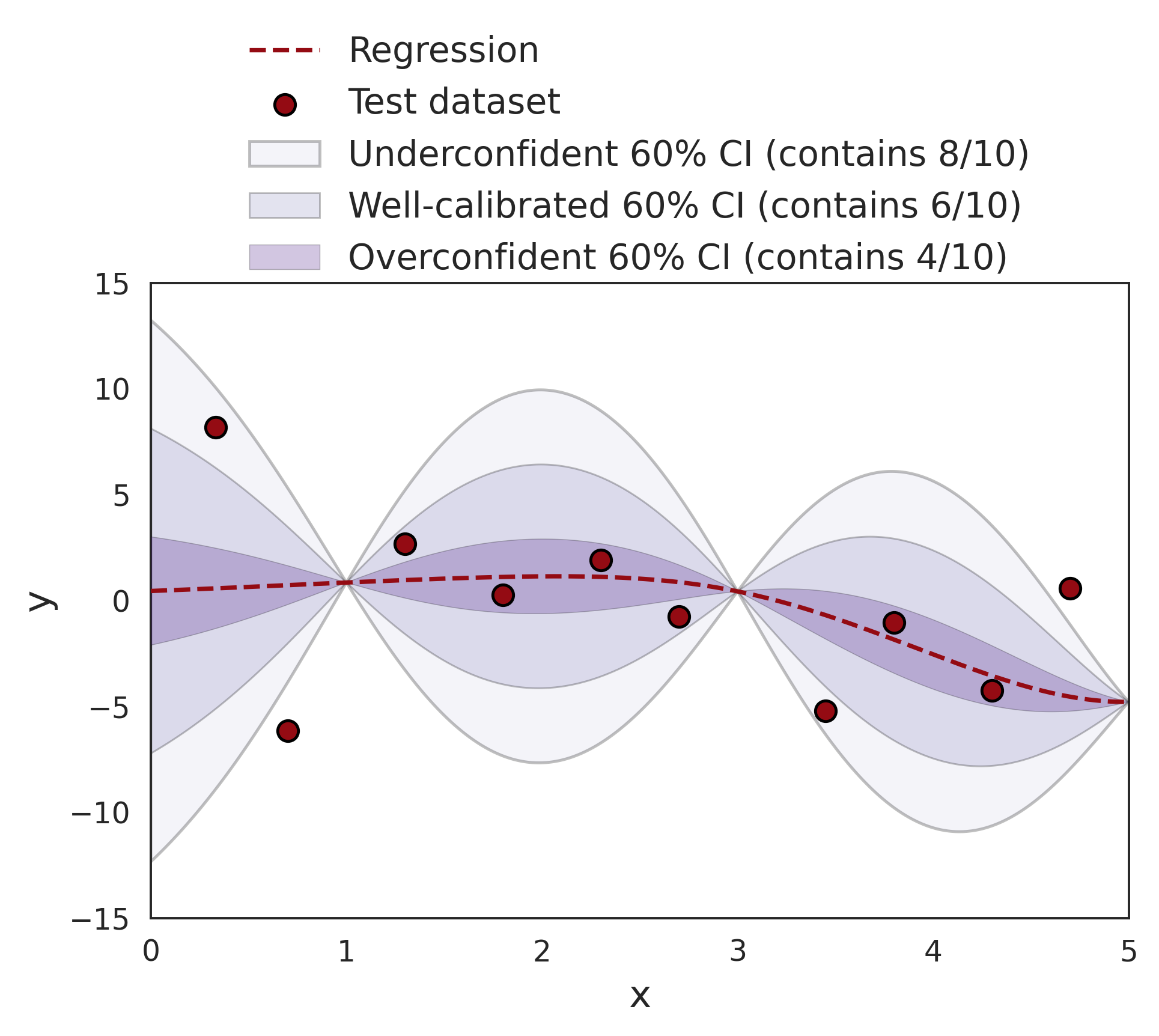}
        \caption{}\label{fig:AUCEa}
    \end{subfigure}
    \hspace{0.0\columnwidth}
    \begin{subfigure}[t]{0.4\textwidth}
        \centering
        \includegraphics[width=\linewidth]{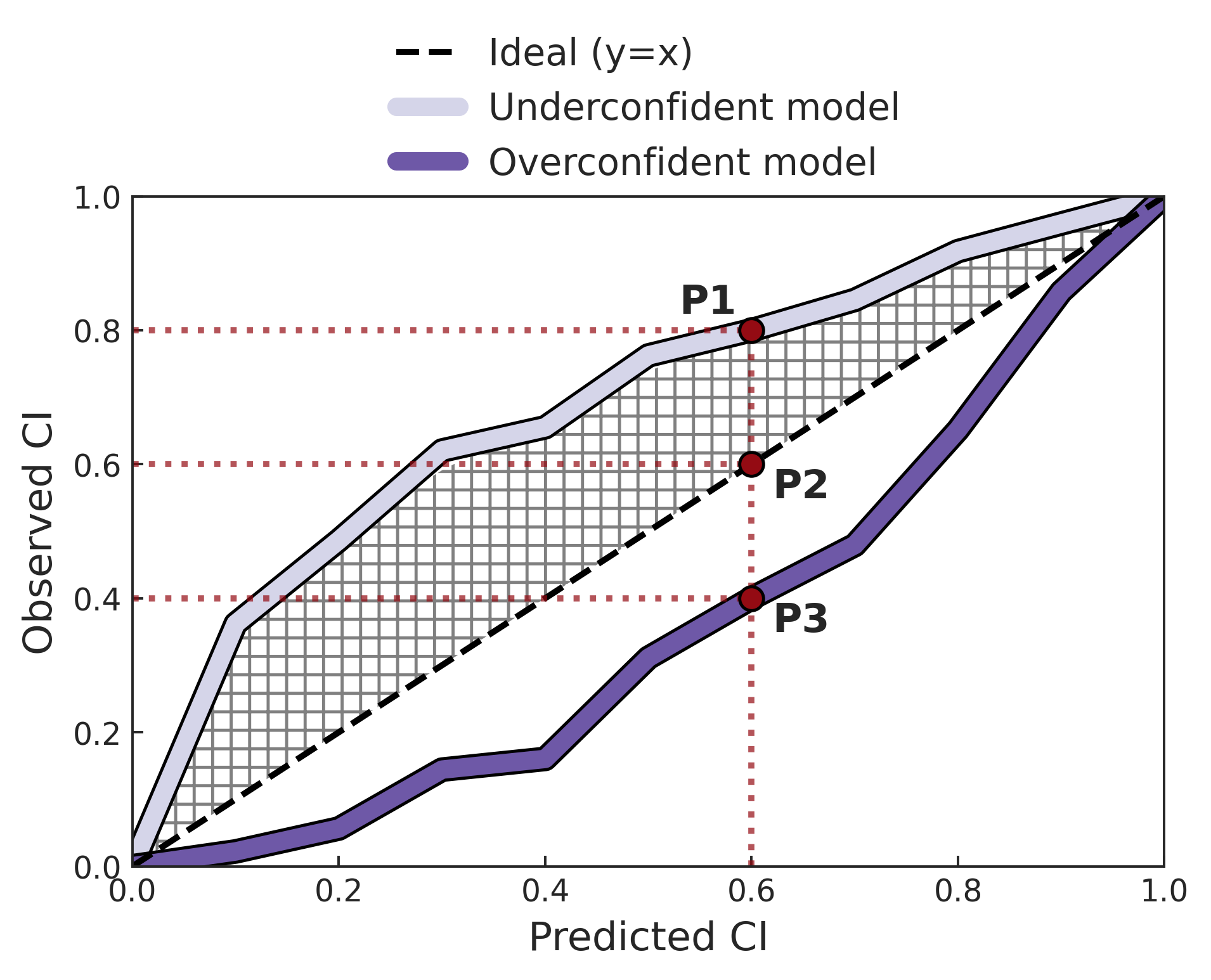}
        \caption{}\label{fig:AUCEb}
    \end{subfigure}\hspace{0.05\textwidth}
    \caption{(a) Illustration of well-calibrated/miscalibrated models: $60\%$ CI of the well-calibrated model contains $60\%$ of the test data, whereas that of the underconfident and overconfident model contains $80\%$ and $40\%$ of the data, respectively. (b) Illustration of CI-based reliability plot.}
    \label{fig:AUCE_info}
\end{figure*} 

The difference between the CI estimated by the model and the actual data it contains can be assessed visually by the CI-based reliability plot shown in Fig. \ref{fig:AUCEb}.  This plot compares the predicted CI from the model on the x-axis with the observed CI measured with the test dataset on the y-axis. To clarify, consider the situation depicted in Fig. \ref{fig:AUCEa}. In the underconfident case, which corresponds to point P1 (x=0.6, y=0.8), the predicted 60\% CI actually corresponds to the observed 80\% CI because 8 out of 10 points are included. Point P2 represents the well-calibrated case, where the predicted 60\% CI by the model matches the actual 60\% of data contained in the CI. In contrast, point P3 represents the overconfident case, where the model predicts a 60\% CI that actually contains only 40\% of the data. In this context, the line $y=x$ represents an ideally well-calibrated model where the predicted CI perfectly matches the observed CI. The algorithm for the CI-based reliability plot is summarized in Algorithm \ref{alg:CI_relia}.

\begin{algorithm}[htb!]
\caption{Procedure for CI-based reliability plot}\label{alg:CI_relia}
\begin{algorithmic}[1]
\State Prepare the test dataset $X$ (with input $x$ and output $y$).
\State Define candidates of CI to be investigated: $P = \{p_1,p_2,...,p_K\}$.
\State {$D = \varnothing$} \Comment{Initialize dataset $D$ to be plotted as y-axis}
\For{$i = 1 : K$} \Comment{Loop for $P$}
    \State {$count = 0$} \Comment{Initialize $count$}
    \State Find $Q(\dfrac{p_i+1}{2}|\mu,\sigma^2)$, which is $\dfrac{p_i+1}{2}$ quantile of $N(\mu,\sigma^2)$.
    \For{$j = 1 : length(X)$} \Comment{Loop for $X$}
        \If{$-Q\bigl(\dfrac{p_i+1}{2}|\mu(x_j),\sigma^2(x_j)\bigr) \leq y_j \leq Q\bigl(\dfrac{p_i+1}{2}|\mu(x_j),\sigma^2(x_j)\bigr)$}
            \State $count += 1$ \Comment{Increase $count$ if test data is within the estimated CI}
        \EndIf
    \EndFor
    \State $\hat{p} = count / length(X)$ \Comment{Calculate observed CI}
    \State $D = D \cup \hat{p}$ \Comment{Append $\hat{p}$ to $D$}
\EndFor
\State Plot CI-based reliability plot: x-axis with $P$ and y-axis with $D$.
\end{algorithmic}
\end{algorithm}

By utilizing this CI-based reliability plot, the AUCE, which is a metric that evaluates the quality of the estimated uncertainty, can be derived. In detail, it is calculated as the area between the ideal line $y=x$ and the reliability plot of the model. The hatched area in Fig. \ref{fig:AUCEb} corresponds to the AUCE of the underconfident model, and the mathematical expression for the AUCE is provided in the following equation \citep{gustafsson2020evaluating}:

\begin{equation}\
\label{eq:AUCE}
\mathrm{AUCE} = \frac{1}{K}\sum\limits_{i=1}^K\lvert\hat{p}-p_i\rvert
\end{equation}
where $K$ refers to the number of CI candidates as in Algorithm \ref{alg:CI_relia}. By definition, a low AUCE value implies that the predictive uncertainty quantified by the model is reliable (or well-calibrated). Additional information on AUCE can be found in \citet{naeini2015obtaining}, \citet{gustafsson2020evaluating}, and \citet{scalia2020evaluating}.

\subsubsection{\label{sec:ENCE} ENCE}

Despite its reputation as a metric of uncertainty quality, AUCE has a critical shortcoming in that it only considers the average over the entire test dataset rather than individuals as mentioned by \citet{levi2022evaluating}. Moreover, they analytically and empirically elaborated that AUCE can be zero even when the predicted distribution is statistically independent from that of the ground truth. In this context, they proposed a novel approach to evaluate the quality of uncertainty, the expected normalized calibration error (ENCE). It was first proposed based on the intuitive assumption: for the well-calibrated model, the estimated uncertainty $\sigma^2(x)$ will be equal to $\bigl(y-\mu(x)\bigr)^2$, MSE. This condition can be expressed mathematically as follows, implying that a higher estimated variance should correspond to a higher expected MSE \citep{phan2018calibrating}:
\begin{equation}\
\label{eq:ENCE_int}
\mathbb{E}_{x,y}[\bigl(y-\mu(x)\bigr)^2 | \sigma^2(x)] = \sigma^2(x)
\end{equation}
The above Eq. \ref{eq:ENCE_int} indicates that the ideally (perfectly) well-calibrated model will have an expected error exactly equal to predictive uncertainty. In this sense, whether the model is well-calibrated can be visually inspected using the error-based reliability plot \citep{scalia2020evaluating,levi2022evaluating}: x-axis as root mean squared error (RMSE), $y-\mu(x)$, and y-axis as root of the mean variance (RMV), $\sigma(x)$. Fig. \ref{fig:ENCE_info} illustrates it, and by its definition in Eq. \ref{eq:ENCE_int}, $y=x$ line indicates the ideally calibrated model. The procedure for its plotting is summarized in Algorithm \ref{alg:err_relia}. 

\begin{figure*}[htb!]
    \centering
            \includegraphics[width=.4\textwidth]{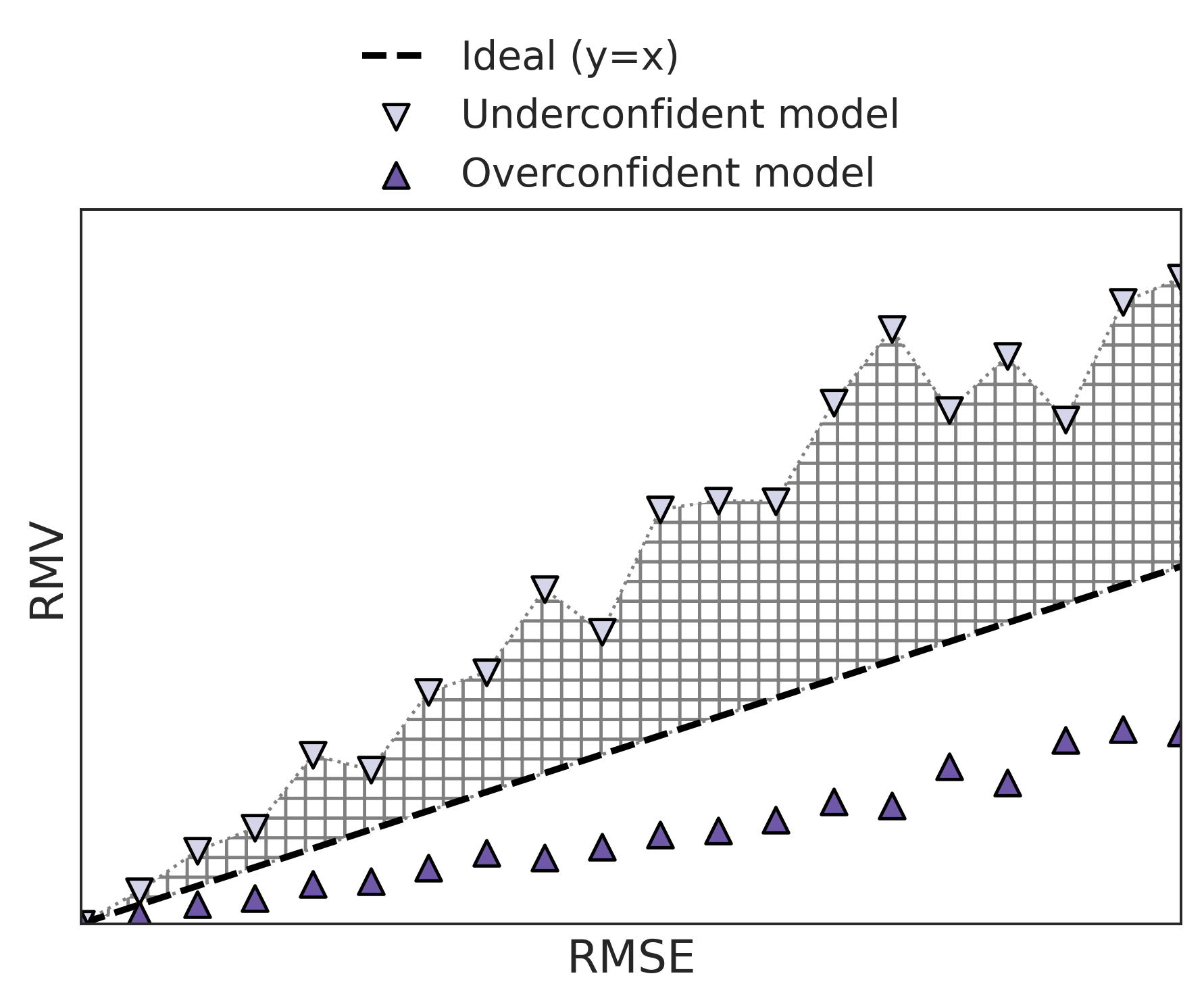}
    \caption{Illustration of error-based reliability plot. Underconfident model overestimates RMV relative to RMSE, while overconfident model underestimates RMV. The ideal model estimates the equivalent RMV and RMSE as the $y=x$ black dashed line.}
    \label{fig:ENCE_info}
\end{figure*} 

\begin{algorithm}[htb!]
\caption{Procedure for error-based reliability plot}\label{alg:err_relia}
\begin{algorithmic}[1]
\State Prepare the test dataset $X$ (with input $x$ and output $y$).
\State Sort $X$ according to $y$ values.
\State Define the number of bins: $B$ (assume $B$ divides $length(X)$).
\State Divide sorted $X$ into $B$ bins, $\tilde{X}=\{\tilde{X}_1,\tilde{X}_2,...,\tilde{X}_B\}$, such that each $\tilde{X}_i$ has the same size of $length(X)/B$.
\State {$D_{RMSE} = \varnothing$} \Comment{Initialize dataset $D_{RMSE}$ to be plotted as x-axis}
\State {$D_{RMV} = \varnothing$} \Comment{Initialize dataset $D_{RMV}$ to be plotted as y-axis}
\For{$i = 1 : B$} \Comment{Loop for $\tilde{X}$}
    \State $D_{RMSE} = D_{RMSE} \cup \sqrt{\dfrac{1}{\lvert \tilde{X}_i \rvert}\sum\limits_{x \in \tilde{X}_i} {\bigl(y(x) - \mu(x)\bigr)^2}}$ \Comment{Append RMSE to $D_{RMSE}$}
    \State $D_{RMV} = D_{RMV} \cup \sqrt{\dfrac{1}{\lvert \tilde{X}_i \rvert}\sum\limits_{x \in \tilde{X}_i} \sigma^2(x)}$ \Comment{Append RMV to $D_{RMV}$}
\EndFor
\State Plot error-based reliability plot: x-axis with $D_{RMSE}$ and y-axis with $D_{RMV}$.
\end{algorithmic}
\end{algorithm}

\clearpage
Then, the area between the ideal $y=x$ line and the error-based reliability plot can be calculated. The normalized version of this value refers to ENCE, the second uncertainty quality metric, and is as follows: 
\begin{equation}\
\label{eq:ENCE}
\mathrm{ENCE} = \frac{1}{B}\sum\limits_{i=1}^B \frac{{\lvert\mathrm{RMV}(i)-\mathrm{RMSE}(i)\rvert}}{\mathrm{RMV}(i)}
\end{equation}
where $B$ indicates the number of bins in Algorithm \ref{alg:err_relia}. Therefore, the ENCE of the underconfident model in Fig. \ref{fig:ENCE_info} can be calculated as the hatched area divided by RMV. As with AUCE, the lower the ENCE value, the better the model is calibrated. 

\subsection{\label{sec:Calib} Uncertainty calibration: STD scaling}

In situations where the estimated uncertainty from the model is imprecise in terms of AUCE (refer to Section \ref{sec:AUCE}) and ENCE (refer to Section \ref{sec:ENCE}), there are various techniques for calibrating uncertainty. Some of these methods include histogram binning \citep{zadrozny2001obtaining}, isotonic regression \citep{zadrozny2002transforming}, and temperature scaling \citep{guo2017calibration}. The first two techniques are non-parametric, and therefore, the number of parameters utilized is dependent on the training dataset size. Conversely, temperature scaling is a parametric approach that needs a fixed number of parameters.

Given the practicality being a crucial consideration in applying UQ techniques to the engineering domain, this research adopts a straightforward approach: temperature scaling. More specifically, the study employs STD scaling, which is a regression task version of temperature scaling \citep{levi2022evaluating}. With STD scaling, it is only necessary to determine a scalar parameter, denoted as $s$, which is used to multiply the standard deviation initially estimated by the DE model, $\hat{\sigma}$. The value of $s$ used in the calibration process is selected to minimize the NLL, as shown below:
\begin{equation}\
\label{eq:NLL_cal}
s=\underset{s}{\operatorname{argmin}}(\cfrac{\mathrm{log}\bigl({s\hat{\sigma}(x)}\bigr)^2}{2}+\cfrac{\bigl(y-\hat{\mu}(x)\bigr)^2}{2\bigl(s\hat{\sigma}(x)\bigr)^2}+\cfrac{\mathrm{log}2\pi}{2}),
\end{equation}
Please note that this equation is the simple modification of Eq. \ref{eq:NLL}, where $\sigma(x)$ is replaced by $s\hat{\sigma}(x)$. This calibration procedure is completely separate from the training procedure of DE; it is performed after the mixture step in Fig. \ref{fig:DE_struct}, so it is called the post-hoc or post-process calibration method. It should be emphasized that the model parameters (weights and biases in the NN model) remain unchanged throughout the calibration process. The STD scaling method is intuitively explained as follows: if the estimated uncertainty from the trained model $\bigl(\hat{\sigma}(x)\bigr)$ is poorly calibrated, the calibrated version of the uncertainty $s\hat{\sigma}(x)$ is used in its place. It is important to note that this calibration process is intended solely to correct the estimated uncertainty, and therefore, only the output $\hat{\sigma}(x)$ of the DE changes, while the predictive value $\hat{\mu}(x)$ remains unaltered. The steps involved in the STD calibration process are outlined in Algorithm \ref{alg:calib}. For calibration, a separate dataset should be used that is distinct from the training and test datasets to ensure calibration generalization \citep{levi2022evaluating}, and therefore, a validation dataset is utilized for the calibration. In multi-output regression tasks, every DE output can be calibrated independently using the number of scaling parameters $s$ equal to the output dimension (this is implemented by the for-loop in line 3 of Algorithm \ref{alg:calib}). In conclusion, this study uses a straightforward STD calibration method for uncertainty calibration, which involves tuning scalar parameters without modifying trained NNs.

\begin{algorithm}[htb!]
\caption{STD calibration procedure}\label{alg:calib}
\begin{algorithmic}[1]
\State Prepare calibration dataset $X$ (with input $x$ and output $y$).
\State Define candidates of scaling factor: $S$
\For {$i = 1 : length(y)$} \Comment{Loop for output dimension of DE}
    \State $s_i=\underset{s \in S}{\operatorname{argmin}}(\cfrac{\mathrm{log}{\bigl(s\hat{\sigma_i}(x)\bigr)^2}}{2}+\cfrac{\bigl(y_i-\hat{\mu_i}(x)\bigr)^2}{2\bigl(s\hat{\sigma_i}(x)\bigr)^2}+\cfrac{\mathrm{log}2\pi}{2})$
\EndFor
\State Utilize $s_i$ to calibrate estimated uncertainty over $i$th output \Comment{Use $s_i\hat{\sigma_i}$ in lieu of $\hat{\sigma_i}$.}
\end{algorithmic}
\end{algorithm}

\section{Application of DE to aerodynamic performance regression task}
\label{sec:DE_missile}

This section applies the DE method to a real-world engineering problem of predicting aerodynamic coefficients for a specific missile configuration with varying flow conditions. It aims to validate the performance of DE in multi-output regression tasks since no comprehensive study has been conducted on this topic. The section evaluates both the regression and uncertainty estimation performance of DE and investigates the impact of $M$, the number of NNs used for the ensemble.

\subsection{\label{sec:pre_missile} Data preparation and training details}

As an engineering problem, the present study adopts the prediction of six aerodynamic coefficients for a particular missile configuration, specifically ``Configuration 1'' described in the NASA TM-2005-213541 report \citep{allen2005aerodynamics}. The Missile Datcom \citep{blake1998missile} low-fidelity semi-empirical solver is then utilized to compute the coefficients for its configuration, given following five flow conditions: $Ma \in [1.1, 3]$, $\phi \in [-90^\circ, 0^\circ]$ (roll angle), $\delta{p} \in [-20^\circ, 20^\circ]$ (pitch control fin deflection angle), $\delta{r} \in [-10^\circ, 0^\circ] $ (roll control fin deflection angle), and $AoA \in [-3^\circ, 23^\circ]$. The resulting aerodynamic coefficients are $C_{NF}$ (normal force coefficient), $C_{AF}$ (axial force coefficient), $C_{PM}$ (pitching moment coefficient), $C_{RM}$ (rolling moment coefficient), $C_{YM}$ (yawing moment coefficient), and $C_{SF}$ (side force coefficient). Subsequently, 9800 points are obtained by full-factorial sampling in the input space and then split into train, validation, and test datasets in the ratio of 8:1:1. The training dataset is utilized to train DE and GPR models, and the validation dataset is used to perform hyperparameter tuning in this section and STD calibration in Section \ref{sec:UQ_calib}, and the test dataset is used for regression and UQ performance evaluation in Sections \ref{sec:pred_results} and \ref{sec:UQ_results}.

After obtaining the dataset, the next step is to determine the structure of the probabilistic NN to be used for the ensemble: following the work by \citet{lakshminarayanan2017simple}, the probabilistic NNs with identical architectures are used for ensembling in this study. To this end, grid search is carried out with hyperparameters regarding the network architecture and the diversity between NNs, such as the number of layers, number of nodes, and size of the mini-batch. Other hyperparameters such as the optimizer algorithm, initial learning rate, and total epochs are selected as Adam, $10^{-3}$, and 13000, respectively. The results of the tuning are available in \ref{sec:app_hyp_DE}. Based on NLL and RMSE, a probabilistic NN with 7 hidden layers and 128 nodes is selected, and the mini-batch size is set to 512. Subsequently, different values of the hyperparameter $M$ (2, 4, 8, and 16) are adopted, with each corresponding DE model referred to as DE-2, DE-4, DE-8, and DE-16 in this manuscript. That is, for DE-16, 16 probabilistic NNs with 7 hidden layers and 128 nodes are trained with a mini-batch size of 512, sharing identical hyperparameters with other DE models except $M$. 

GPR models with different kernels are also trained for their hyperparameter tuning. To this end, Mat\'ern 5/2, radial basis function, rational quadratic, and dot-product kernels are examined \citep{williams2006gaussian}, and their results also can be found in \ref{sec:app_hyp_GPR}. In addition, not only single-output GPR models are tested, but also the multi-output GPR (MOGPR) with radial basis function kernel is trained for more comprehensive comparison with DE \citep{williams2006gaussian, alvarez2012kernels, lin2021multi, lin2022gradient}. Among them, GPR with Mat\'ern 5/2 kernel shows the best performance, and is therefore selected for the comparison with DE throughout this paper. The required training times for DE with selected NN architecture (7 hidden layers, 128 nodes, minibatch size of 512) and GPR (Mat\'ern 5/2 kernel) models using Intel(R) Xeon(R) CPU @ 2.20GHz are as follows: 10.9 hours for GPR, 2.4 hours for DE-2, 5 hours for DE-4, 9.7 hours for DE-8, and 19.4 hours for DE-16. Note that GPR requires more training time than DE for both the hyperparameter tuning and final model training.

\clearpage
\subsection{\label{sec:pred_results} Evaluation of regression performance}

In this section, the regression performances of all selected models are presented using a test dataset that is not used in model training. Before going into details, DE-2 (which required the least training time among the DE models) is compared with GPR to highlight the efficiency of the DE models. Fig. \ref{fig:kde_plot} shows the results of kernel density estimation (KDE), which demonstrates the generalization performance of the models by visualizing the distributions of the test data in terms of NLL and RMSE (those of all six QoIs are averaged to be shown in this figure). For both criteria, the obvious superiority of DE-2 can be identified: most of the test data is concentrated in the lower error region in DE-2. More specifically, the KDE of NLL shows that the density peak of DE-2 represented by a star with long dashed line is located at NLL of -4.6, while that of GPR is located at -3.1. When it comes to RMSE, the peak of DE-2 is at RMSE of 0.003 while GPR is at 0.09. The medians of the error metrics are also shown as circles with dotted lines. For both metrics, those of DE-2 are much lower than those of GPR, indicating that DE-2 performs better than GPR overall. The most interesting point here is that although DE-2 requires only 22\% of the training time of GPR, it achieves superior regression accuracy.

\begin{figure}[htb!]
    \centering
    
        \includegraphics[width=.45\columnwidth]{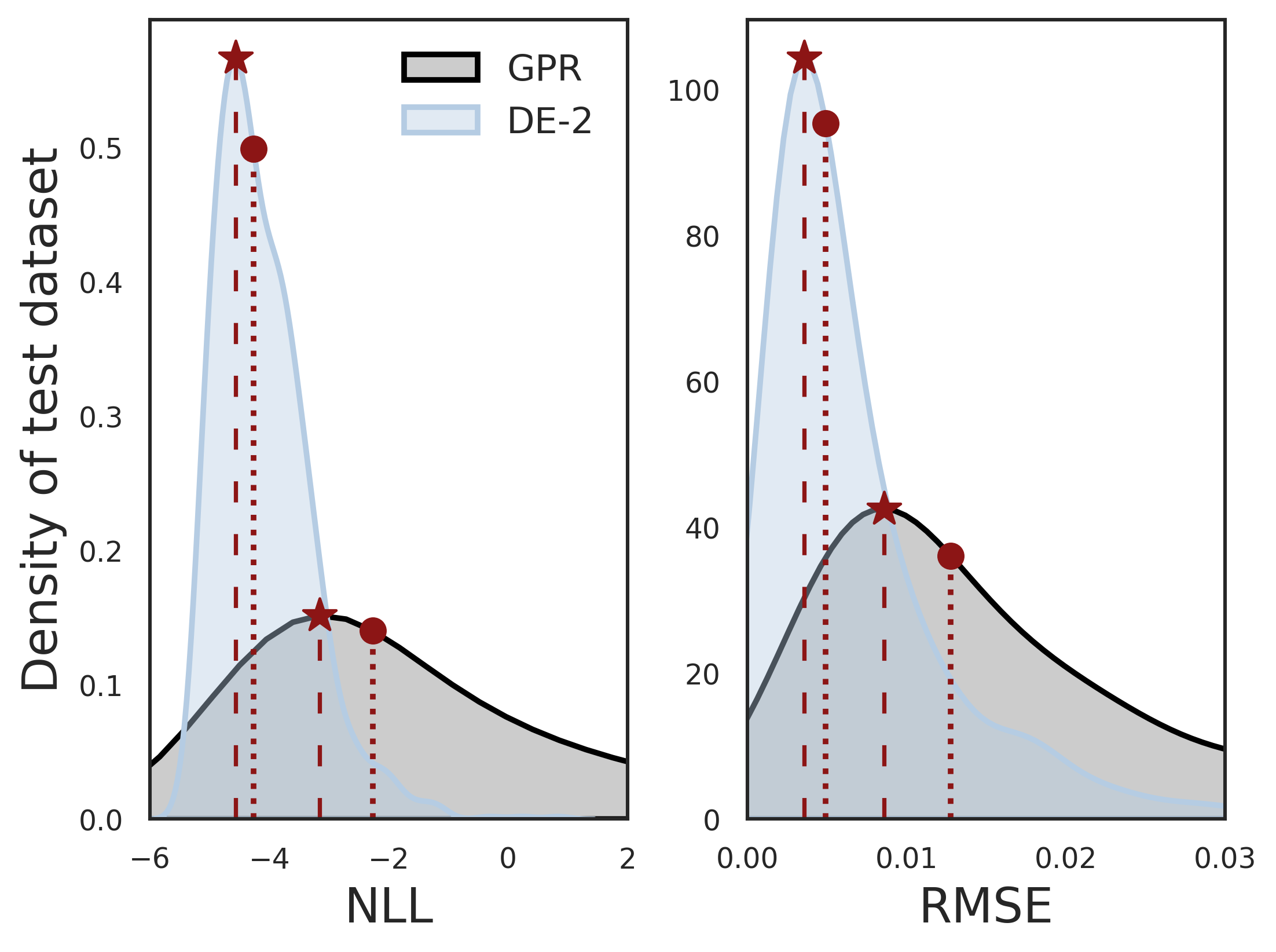}
        
    \caption{Comparison of regression accuracy between GPR and DE-2: kernel density estimation (KDE) of test dataset with respect to NLL and RMSE (averaged values of all six QoIs). The stars and circles represent the maximum and median points of each model, respectively.}
    \label{fig:kde_plot}
\end{figure} 

Fig. \ref{fig:pred_compar} provides the comprehensive results of the regression performance. Fig. \ref{fig:pred_compar_a} shows the NLL results of all models with respect to the six aerodynamic QoIs, and their averaged NLL is also shown at the right end. Throughout all QoIs, GPR shows inferior regression accuracy than all other DE models. The results on NLL could be expected as each NN in the DE model is trained to minimize NLL. However, the results on RMSE in Fig. \ref{fig:pred_compar_b} are highly inspiring: they also achieve higher regression accuracy even in terms of RMSE. Considering that numerous engineers use RMSE to evaluate regression models, the fact that the average RMSE of DE models is less than half that of GPR is quite encouraging. Also, contrary to the claim that DE-5 would be sufficient in the work first proposed DE approach \citep{lakshminarayanan2017simple}, DE-2 seems to be sufficient enough in this study, at least in terms of predictive accuracy: the difference in their values between all DE models is insignificant. However, the conventional belief is that the more models used in the ensemble, the more accurate the prediction will be due to the robustness that comes from averaging multiple predictions. The underlying reason for this counter-intuitive result (that is, insignificant differences in predictive accuracy as $M$ increases) is thought to be the insufficient diversity within individual models due to the strategy adopted by DE: identical dataset and model architecture \citep{lakshminarayanan2017simple}. However, note that blindly ensuring excessive diversity by using different datasets and model architectures should be done with caution, since it can degrade UQ performance (which will be shown in \textbf{Remark 2} of Section \ref{sec:math}). In this regard, a trade-off study between the predictive accuracy and uncertainty quality as the diversity varies within individual NNs may be an interesting future work.

\begin{figure*}[htb!]
    \centering
    \begin{subfigure}[h]{0.45\textwidth}
        \centering
        \includegraphics[width=\linewidth]{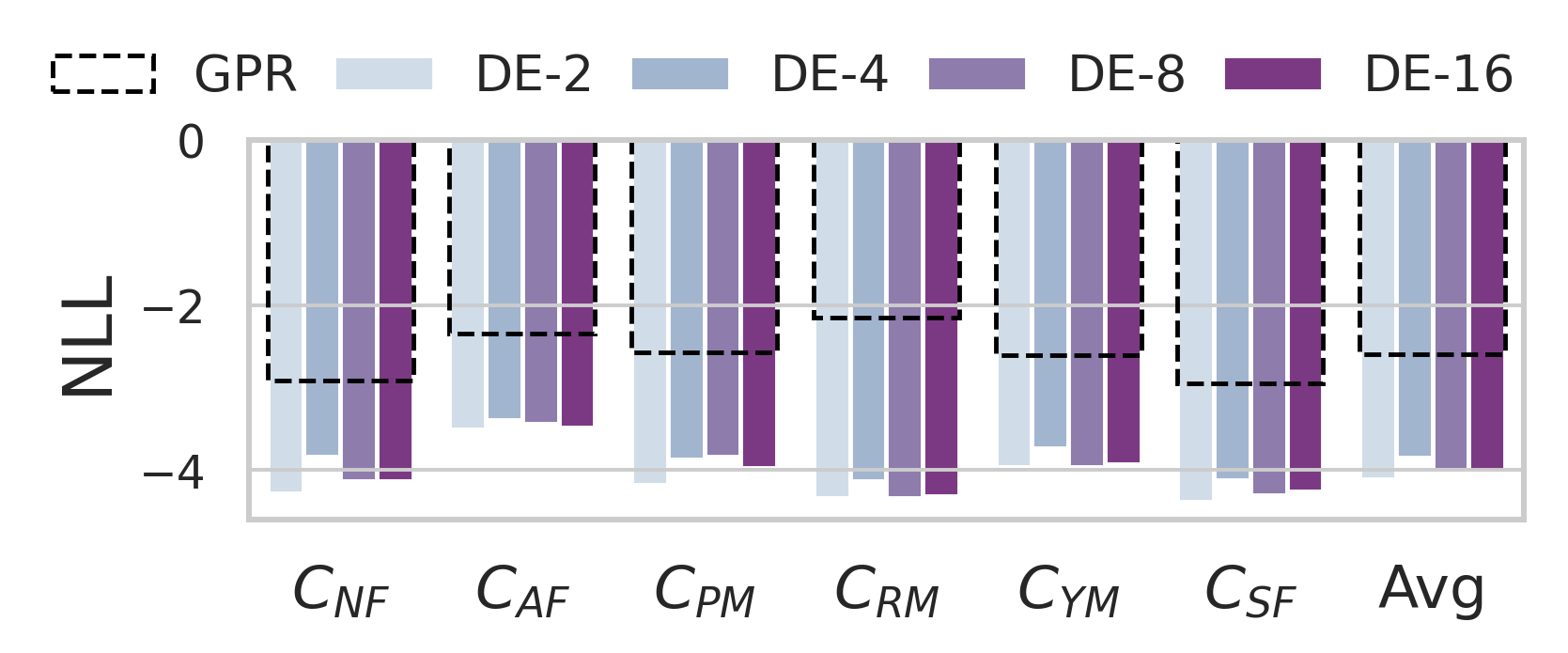}
        \caption{}\label{fig:pred_compar_a}
    \end{subfigure}
    \vfill
    \begin{subfigure}[h]{0.45\textwidth}
        \centering
        \includegraphics[width=\linewidth]{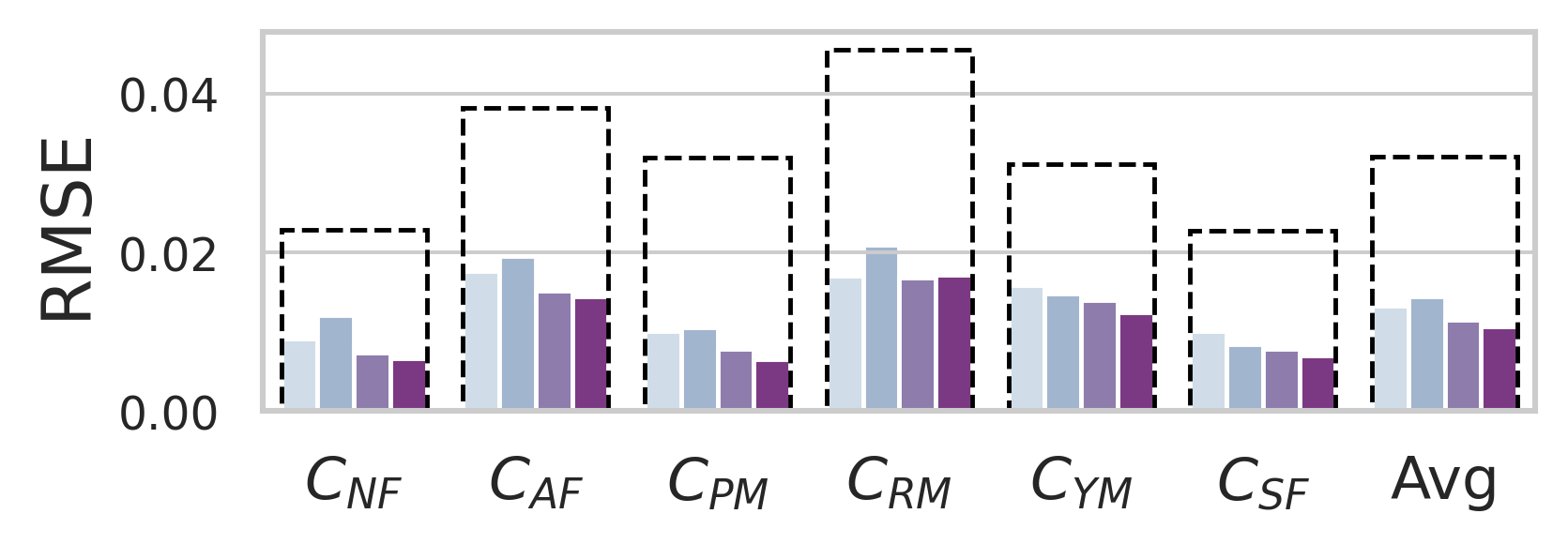}
        \caption{}\label{fig:pred_compar_b}
    \end{subfigure}
    \caption{Comparison of regression accuracy between GPR and all DE models: comprehensive results in terms of all aerodynamic QoIs. (a) NLL, (b) RMSE.}
    \label{fig:pred_compar}
\end{figure*}

% \clearpage
\subsection{\label{sec:UQ_results} Evaluation of UQ performance}

This section examines the quality of the predictive uncertainty, using AUCE and ENCE criteria for the quantitative investigation. For this purpose, reliability plots should be drawn first, using the test dataset split in Section \ref{sec:pre_missile} (dataset size of 980). Also, as in Algorithm \ref{alg:CI_relia}, CI-based reliability plots require the set of CI candidates ($P$) and error-based reliability plots in Algorithm \ref{alg:err_relia} need the number of bins ($B$). In this study, $P = \{0.1, 0.2, ..., 0.9\}$ and $B = 20$ are chosen.

Fig. \ref{fig:reliab_GPR} shows the results of GPR, and it appears that GPR has a satisfactory uncertainty quality with respect to the error-based reliability plot (Fig. \ref{fig:reliab_GPR_b}), while the CI-based plot (Fig. \ref{fig:reliab_GPR_a}) shows relatively poor quality. In a CI-based plot, since the predicted CI (x-axis) is underestimated compared to the actual observed CI (y-axis), it can be inferred that GPR is trained to be ``underconfident'': it is underconfident itself, so it overestimates its uncertainty.

\begin{figure}[htb!]
    \centering
    \begin{subfigure}[h]{0.4\textwidth}
        \centering
        \includegraphics[width=\linewidth]{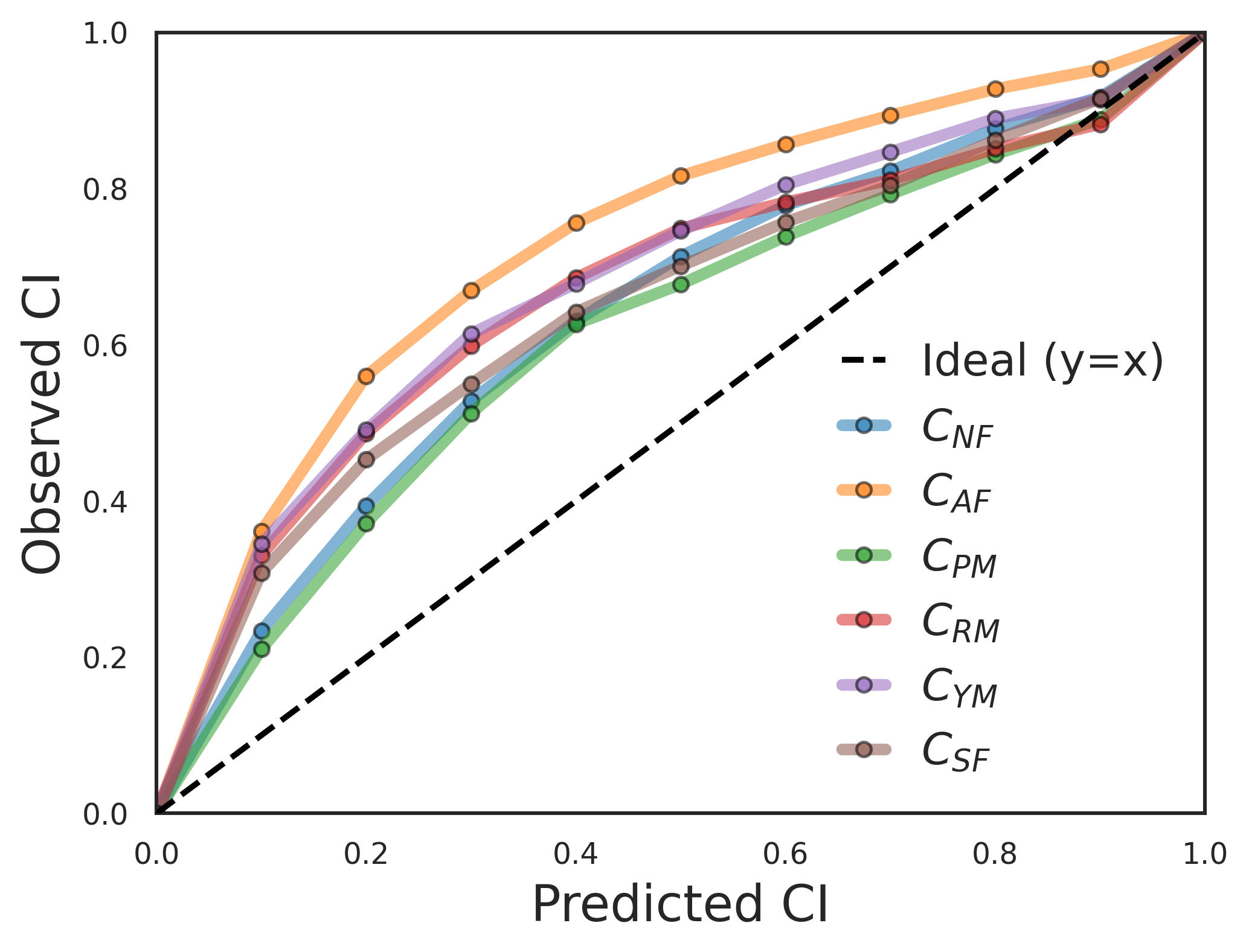}
        \caption{}\label{fig:reliab_GPR_a}
    \end{subfigure}
    \hspace{0.0\columnwidth}
    \begin{subfigure}[h]{0.4\textwidth}
        \centering
        \includegraphics[width=\linewidth]{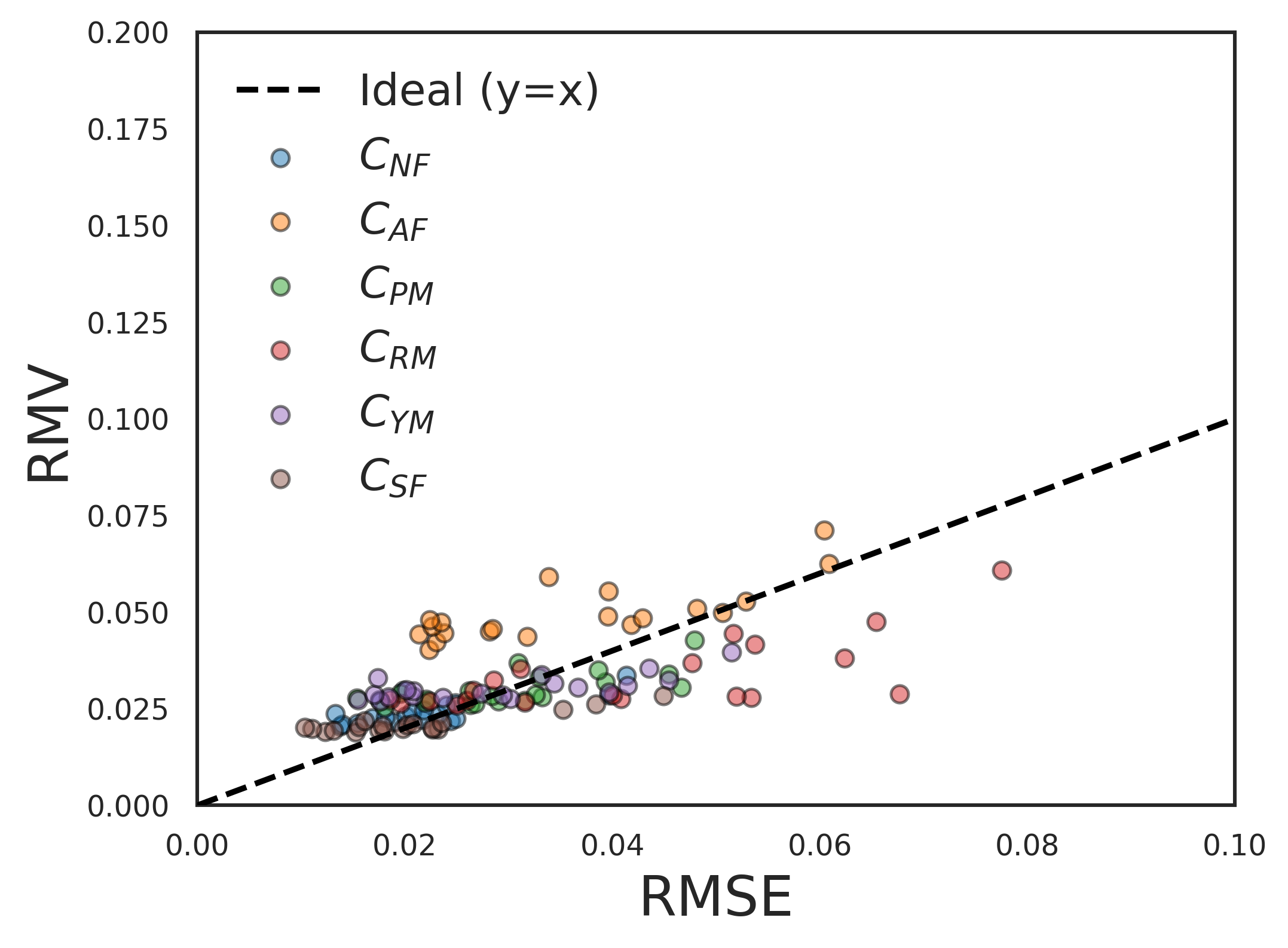}
        \caption{}\label{fig:reliab_GPR_b}
    \end{subfigure}
    \caption{Reliability plots of GPR: (a) CI-based reliability plot, (b) Error-based reliability plot.}\label{fig:reliab_GPR}
\end{figure} 

The results of the DE models are then shown in Fig. \ref{fig:reliab_DE_bef}. Note that unlike GPR in Fig. \ref{fig:reliab_GPR}, only the results of output $C_{SF}$ are visualized to highlight the differences between DE models: comprehensive results can be found in Fig. \ref{fig:reliab_DE_comprehensive} in \ref{sec:app_UQ_bef}. For DE-2, the CI-based plot (Fig. \ref{fig:reliab_DE_bef_a}) shows a similar trend to that of GPR, while the error-based plot (Fig. \ref{fig:reliab_DE_bef_b}) shows slightly better quality. Meanwhile, a notable trend is observed along the increase of $M$: as it increases, the uncertainty quality with respect to both reliability plots apparently degrades. More specifically, both types of plots move upward away from the $y=x$ ideal line as $M$ increases, indicating that DE models tend to become ``underconfident''. Considering that DE-16 requires about 8 times as much training time as DE-2, it can be confirmed that using large $M$ values for the ensemble does not necessarily lead to better results, but rather the opposite in terms of uncertainty quality. In this context, assuming that the performance of DE-5 will be between DE-4 and DE-8, it can be inferred that using $M=5$ as suggested by \citet{lakshminarayanan2017simple} does not guarantee sufficient UQ quality in this case. In fact, the insight behind this underconfident tendency when ensembling networks in classification tasks can be found in \citet{rahaman2021uncertainty}, while the corresponding tendency in regression has not been proven. Accordingly, in the next section, we provide the mathematical explanation for this underconfident tendency in regression tasks.

\begin{figure}[htb!]
    \centering
    \begin{subfigure}[h]{0.45\textwidth}
        \centering
        \includegraphics[width=\linewidth]{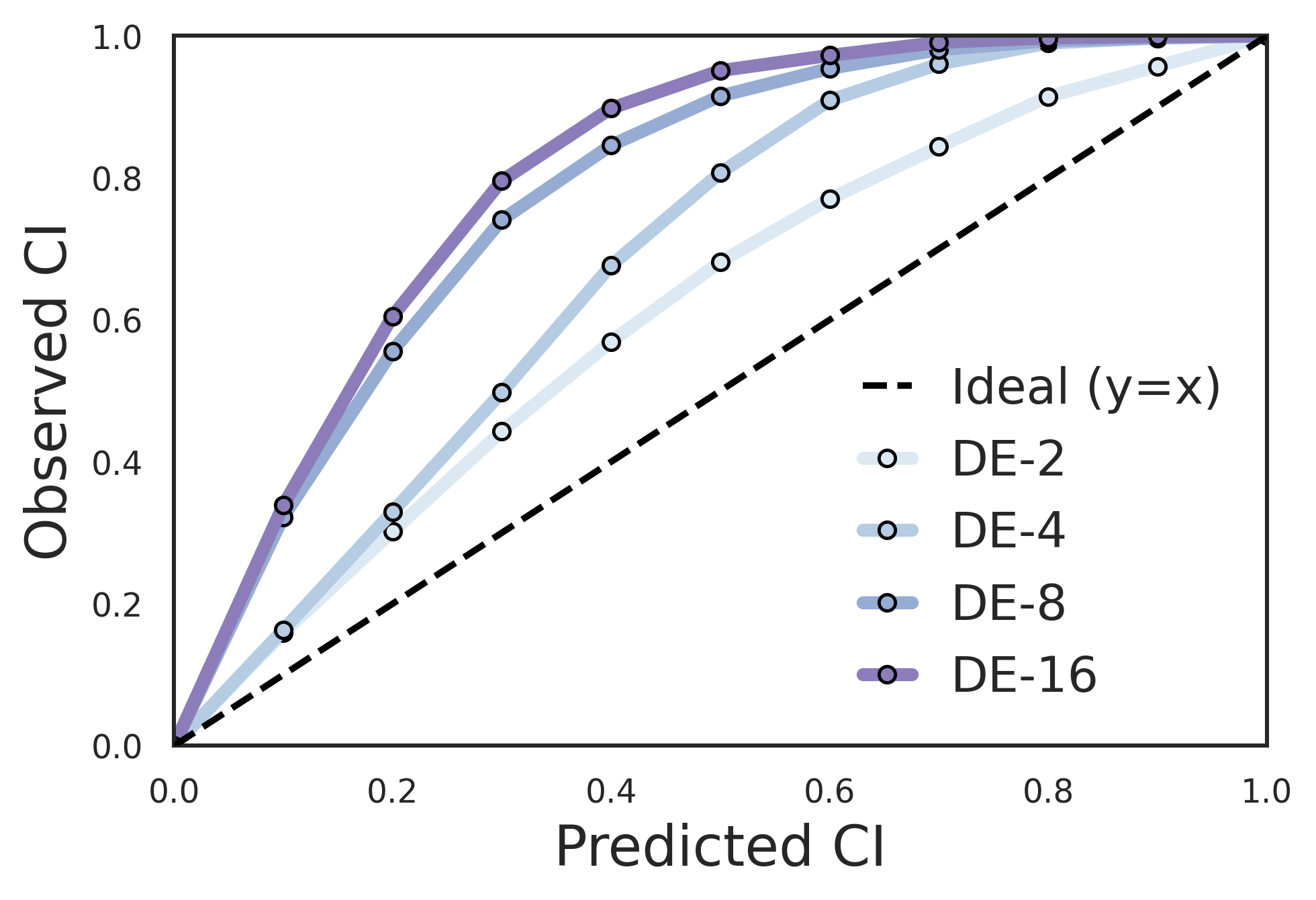}
        \caption{}\label{fig:reliab_DE_bef_a}
    \end{subfigure}
    \hspace{0.0\columnwidth}
    \begin{subfigure}[h]{0.45\textwidth}
        \centering
        \includegraphics[width=\linewidth]{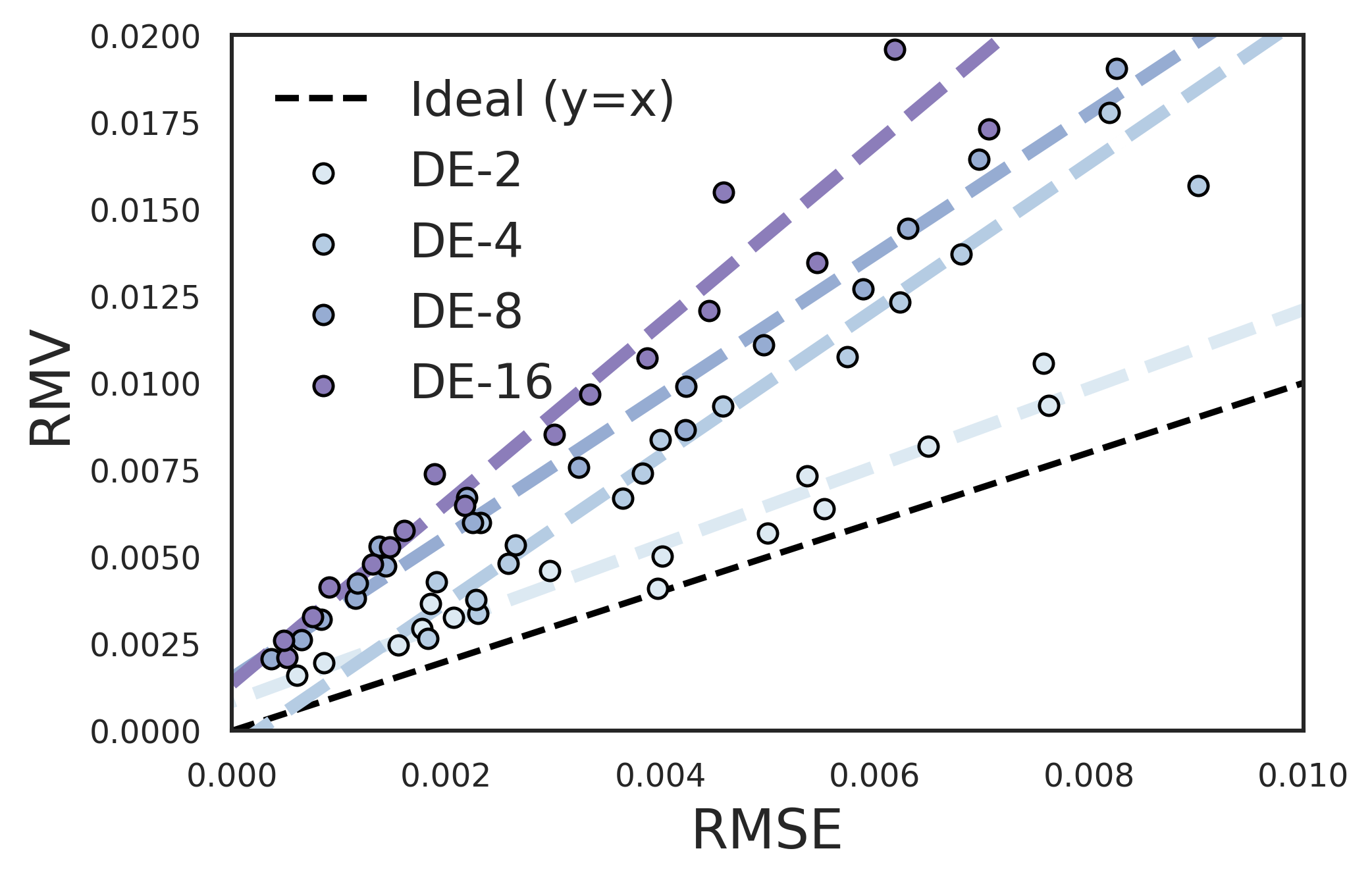}
        \caption{}\label{fig:reliab_DE_bef_b}
    \end{subfigure}
    \caption{Reliability plots of DE: for simplicity, only the $C_{SF}$ results of different DE models are shown. (a) CI-based reliability plot, (b) Error-based reliability plot. In (b), to clearly show the decreasing tendency of UQ quality with increasing $M$, the linear regression model of the scatter points of each DE model is shown as a dashed line with the corresponding color.}\label{fig:reliab_DE_bef}
\end{figure}

\clearpage
\subsection{\label{sec:math} Theoretical derivation: underconfidence of DE in regression tasks}

This section is for the mathematical derivation of why the ensemble of NNs becomes underconfident, as discovered in the previous section. For this purpose, the deviation from calibration (DC) score is introduced as in \citet{rahaman2021uncertainty}; their work focused only on classification tasks, so their DC score consisted of the Brier score and the entropic term. Meanwhile, since our work focuses on the regression task, we adopted the different DC score consisting of the MSE and predictive variance. In this context, the following proposition and its proof can be considered as one of the contributions of this paper.

\begin{prop}
When DC score is defined as follows,

\begin{equation}
DC(\mu, \sigma) \equiv (y - {\mu})^2 - {\sigma}^2
\end{equation}

DC score of the ensemble becomes less than or equal to the averaged DC score of the individual NNs.

\begin{equation}
\label{eq:proposition}
DC(\hat{\mu}, \hat{\sigma}) \leq \cfrac{1}{M} \sum\limits_{i=1}^MDC(\mu_i, \sigma_i)
\end{equation}

\end{prop}

\begin{proof}
The averaged DC score of the individual NNs (right-hand side of the Eq. \ref{eq:proposition}) can be expressed as:

\begin{equation}
\begin{aligned}[b]
\cfrac{1}{M} \sum\limits_{i=1}^M {DC(\mu_i, \sigma_i)}
& = \cfrac{1}{M} \sum\limits_{i=1}^M (y^2 - 2y{\mu_i} + {\mu_i}^2 - {\sigma_i}^2) \\
& = y^2 - 2y\hat{\mu} + \cfrac{1}{M} \sum\limits_{i=1}^M {{\mu_i}^2} - \cfrac{1}{M} \sum\limits_{i=1}^M {{\sigma_i}^2} \\
& = (y^2 - 2y\hat{\mu} + \hat{\mu}^2 - \hat{\sigma}^2) + (\cfrac{1}{M} \sum\limits_{i=1}^M {{\mu_i}^2} - \hat{\mu}^2) + (\hat{\sigma}^2 - \cfrac{1}{M} \sum\limits_{i=1}^M {{\sigma_i}^2}) \\
& = \underbrace{(y^2 - 2y\hat{\mu} + \hat{\mu}^2 - \hat{\sigma}^2)}_{= DC(\hat{\mu}, \hat{\sigma})} + 2 \underbrace{(\cfrac{1}{M} \sum\limits_{i=1}^M {{\mu_i}^2} - \hat{\mu}^2)}_{= Var(\mu_i)} \quad (\because \text{Eq. \ref{eq:mixture_s}})
\end{aligned}
\end{equation}
 % && (\because \text{Eq. \ref{eq:mixture_s}})
Hence,

\begin{equation} \label{eq:proof_last}
\begin{aligned}
DC(\hat{\mu}, \hat{\sigma}) 
= \cfrac{1}{M} \sum\limits_{i=1}^MDC(\mu_i, \sigma_i) - \underbrace{2 \cdot Var(\mu_i)}_{\geq 0} \\
\end{aligned}
\end{equation}

\end{proof}

\begin{remark1}\label{remark1}
The DC used in the above proposition indicates the degree of calibration. When DC equals 0, it means that the estimated uncertainty $\sigma^2$ exactly matches the MSE, $(y - {\mu})^2$. If $DC<0$, the uncertainty is overestimated compared to the MSE, which is an underconfident case. Therefore, the proposition that the DC score decreases after ensembling has mathematically explained the underconfidence of DE models observed in Section \ref{sec:UQ_results}.  
\end{remark1}

\begin{remark2}\label{remar2}
In Section \ref{sec:pred_results}, it was mentioned that introducing excessive diversity to individual NNs can lead to degraded UQ performance. This can be easily inferred by the term $Var(\mu_i)$ in Eq. \ref{eq:proof_last}: the more variance NNs have, the more underconfidence their ensemble shows.
\end{remark2}

\clearpage
\section{\label{sec:calib_results} DE models with STD calibration}

The underconfidence tendency of DE models in regression tasks is observed and explained in the previous section. This section suggests the use of post-hoc STD calibration to mitigate this undesirable tendency and examines its effects.

\subsection{\label{sec:UQ_calib} STD calibration of DE models}

The findings presented in Section \ref{sec:UQ_results} suggest that, despite the prevailing view that DE models are well-calibrated, this is not always the case, as illustrated in this straightforward multi-output regression task within an engineering domain. To address this issue, we propose using the STD calibration method on the trained DE models. This technique, as described in Algorithm \ref{alg:calib}, is straightforward and practical, as it requires only a single for-loop and leverages the existing models without additional training. This makes it a feasible option for our study, which focuses on the application of DE in engineering, where practicality is crucial.

Algorithm \ref{alg:calib} first requires a set of candidates for scaling factors, $S$. Since the scaling factor of 1 corresponds to the case without calibration, the candidates $s$ are set around 1. Accordingly, $s=10^x$ are chosen as candidates, where $x$ are 100 uniformly distributed points from -2 to 0.18, so that the resulting range of scaling factors to explore is from 0.01 to 1.5. Note that with $s$ less than 1, underconfident models that overestimate the standard deviations (uncertainty) can be calibrated. Finally, the STD calibration is performed using validation dataset split in Section \ref{sec:pre_missile} (dataset size of 980) and the optimized scaling factors for each DE model with respect to each output (QoI) are summarized in Table \ref{tab:cal_results}. The STD calibration for all models is performed within 60 seconds, which is negligible compared to their training time.

\renewcommand{\arraystretch}{1.1}
\begin{table}[htb!]
\caption{Optimized scaling factors for STD calibration.} \label{tab:cal_results}
    \begin{tabular*}{0.95\columnwidth}{@{\extracolsep{\fill}}lccccccc}
        \hline 
        \multirow{2}{*}{Methods} & \multicolumn{7}{c}{Optimized scaling factors} \\ 
        \cline{2-8}
        & $C_{NF}$ & $C_{AF}$ & $C_{PM}$ & $C_{RM}$ & $C_{YM}$ & $C_{SF}$ & \textbf{Avg} \\ \hline
        DE-2 & 0.549 & 1.061 & 0.608 & 1.009 & 0.824 & 0.578 & \textbf{0.771}\\
        DE-4 & 0.385 & 0.405 & 0.284 & 0.472 & 0.257 & 0.270 & \textbf{0.345}\\
        DE-8 & 0.147 & 0.270 & 0.133 & 0.270 & 0.155 & 0.140 & \textbf{0.186}\\
        DE-16 & 0.103 & 0.199 & 0.088 & 0.189 & 0.120  & 0.108 & \textbf{0.135}\\ \hline
    \end{tabular*}

\end{table}

Herein, the scaling factors for all six aerodynamic coefficients and their average value in each model are presented. The most notable point is that almost all $s$ values are less than 1 and they decrease as $M$ increases: see the bold values in Table \ref{tab:cal_results} to confirm their average trend. Taken together with the results from Section \ref{sec:UQ_results} that DE models overestimate their $\sigma^2$ (become underconfident) as $M$ increases, one might expect optimized $s<1$ to mitigate this underconfident tendency. And Fig. \ref{fig:reliab_DE_aft} proves that this actually happens: reliability plots of the DE models after STD calibration are drawn with the test dataset. Note that the validation dataset used during the STD calibration should not be reused in this process for generalization purposes. When compared to the previous plots in Fig. \ref{fig:reliab_DE_bef}, the obvious improvement due to the calibration technique can be observed. See \ref{sec:app_UQ_aft} for comprehensive results on the calibration effects with respect to all six QoIs.

\begin{figure}[htb!]
    \centering
    \begin{subfigure}[h]{0.45\textwidth}
        \centering
        \includegraphics[width=\linewidth]{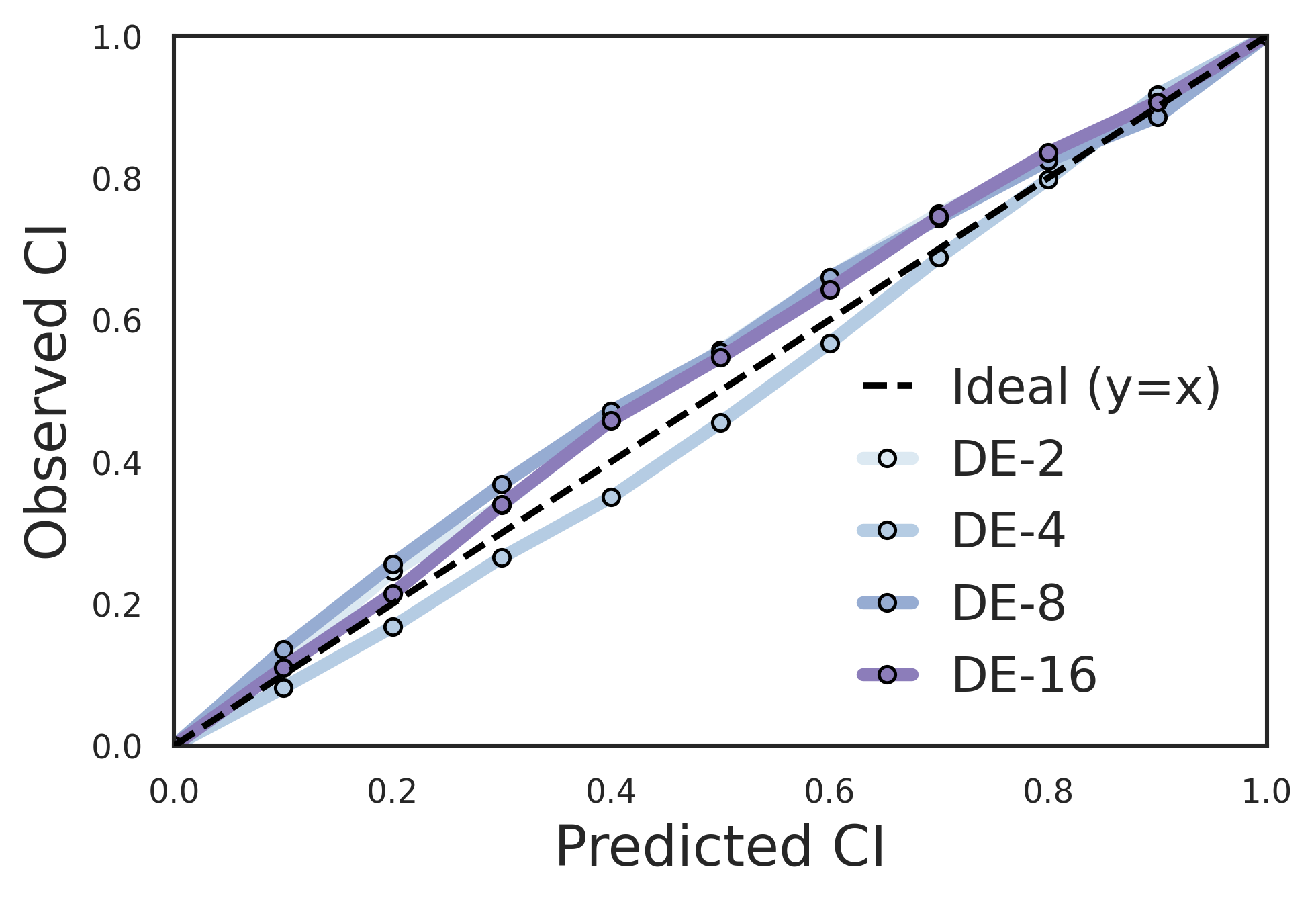}
        \caption{}\label{fig:reliab_DE_aft_a}
    \end{subfigure}
    \hspace{0.0\columnwidth}
    \begin{subfigure}[h]{0.45\textwidth}
        \centering
        \includegraphics[width=\linewidth]{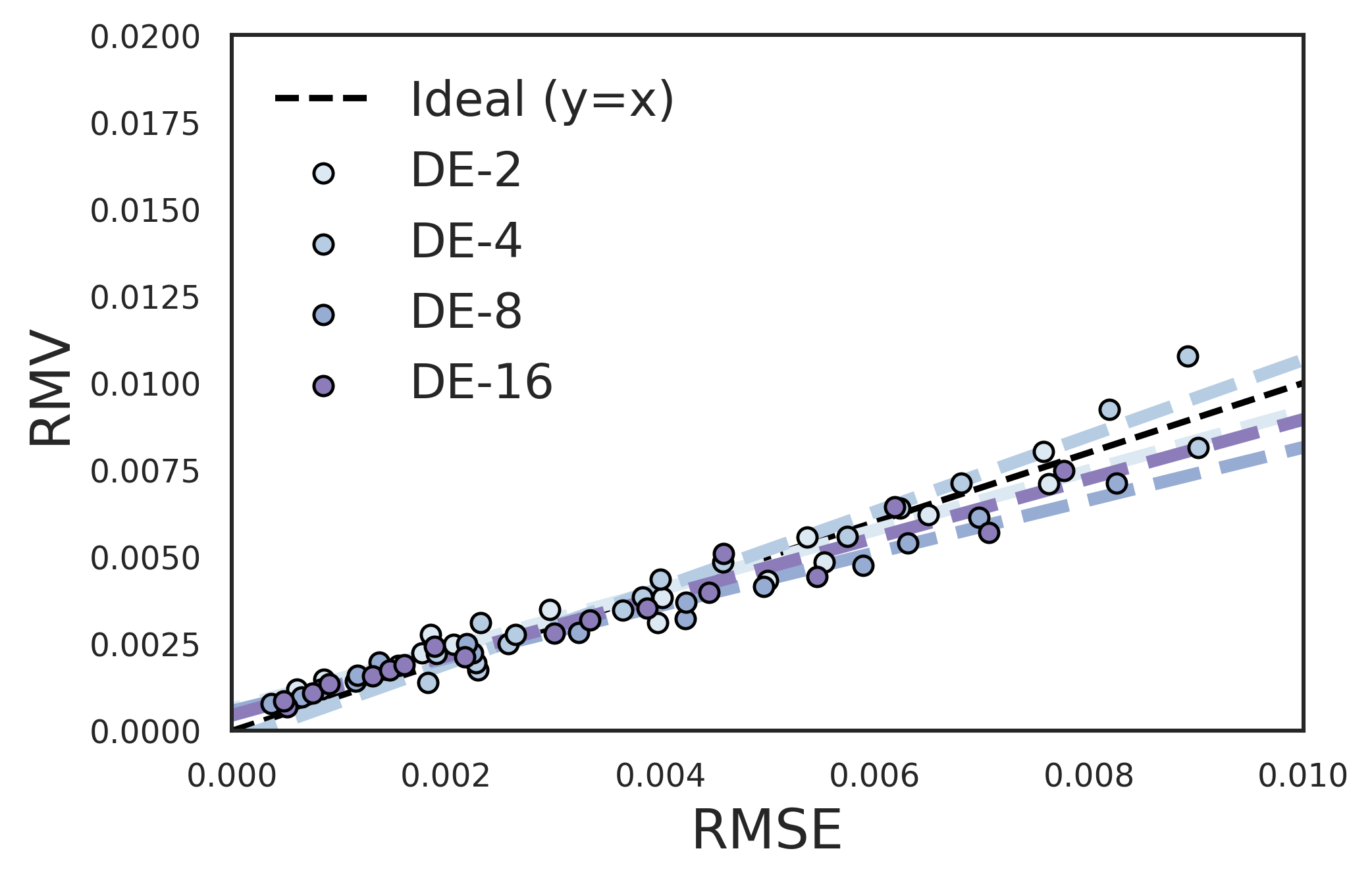}
        \caption{}\label{fig:reliab_DE_aft_b}
    \end{subfigure}
    \caption{Reliability plots of DE after STD calibration: (a) CI-based reliability plot, (b) Error-based reliability plot. The noticeable effects of STD calibration can be found when compared with the corresponding figure before STD calibration, Fig. \ref{fig:reliab_DE_bef}.}
    \label{fig:reliab_DE_aft}
\end{figure} 

Then, the quantitative effects of the calibration in terms of AUCE and ENCE will be analyzed, and from now on DE before and after calibration will be referred to as DE-bef and DE-aft, respectively. The AUCE and ENCE of the GPR will also be presented for the comparison, but please note that the GPR can be considered inherently STD-calibrated since its training algorithm already aims to minimize NLL as in the STD calibration process. This means that the GPR does not require additional STD calibration for a fair comparison with DE-aft because it can be seen as having already undergone STD calibration. Finally, the results are summarized in Fig. \ref{fig:slope_plot}. It consists of the sub-figures, where the row indicates each UQ metric, the column indicates each QoI, and the x-axis in each sub-figure indicates whether the DE undergoes STD calibration (as explained, GPR metrics have a constant value along the x-axis regardless of the STD calibration). Before the calibration, the AUCE (upper row) of GPR is between DE models: DE-2 is better than GPR, DE-4 is similar, and DE-8 and DE-16 are worse. However, the STD calibration completely changes this situation: AUCE of all DE models for all aerodynamic QoIs decreases dramatically. For all outputs, DE-aft clearly outperforms GPR. The significant improvement of ENCE (lower row) due to the calibration of DE can also be verified. DE models show worse performance than GPR without calibration, but this gap narrows and even reverses, as seen in the rightmost subplot ``Avg'', which shows the average performance of all outputs. At least in our study, DE-2 with STD calibration can be regarded as the best model, since DE models show negligible differences in predictive accuracy, UQ quality after calibration, while DE-2 requires the shortest training time (note that this does not mean that the lower $M$ is better; careful consideration of $M$ in each situation is a prerequisite, as in this study). In summary, vanilla DE outperformed GPR in terms of training efficiency and regression accuracy, but not in terms of quality of estimated uncertainty. However, when used with a simple post-hoc STD calibration (which requires negligible additional post-processing time), DE demonstrated its strong potential as an alternative to GPR in terms of training time, prediction accuracy, and also UQ quality.

\begin{figure*}[htb!]
    \centering
        \includegraphics[width=0.95\textwidth]{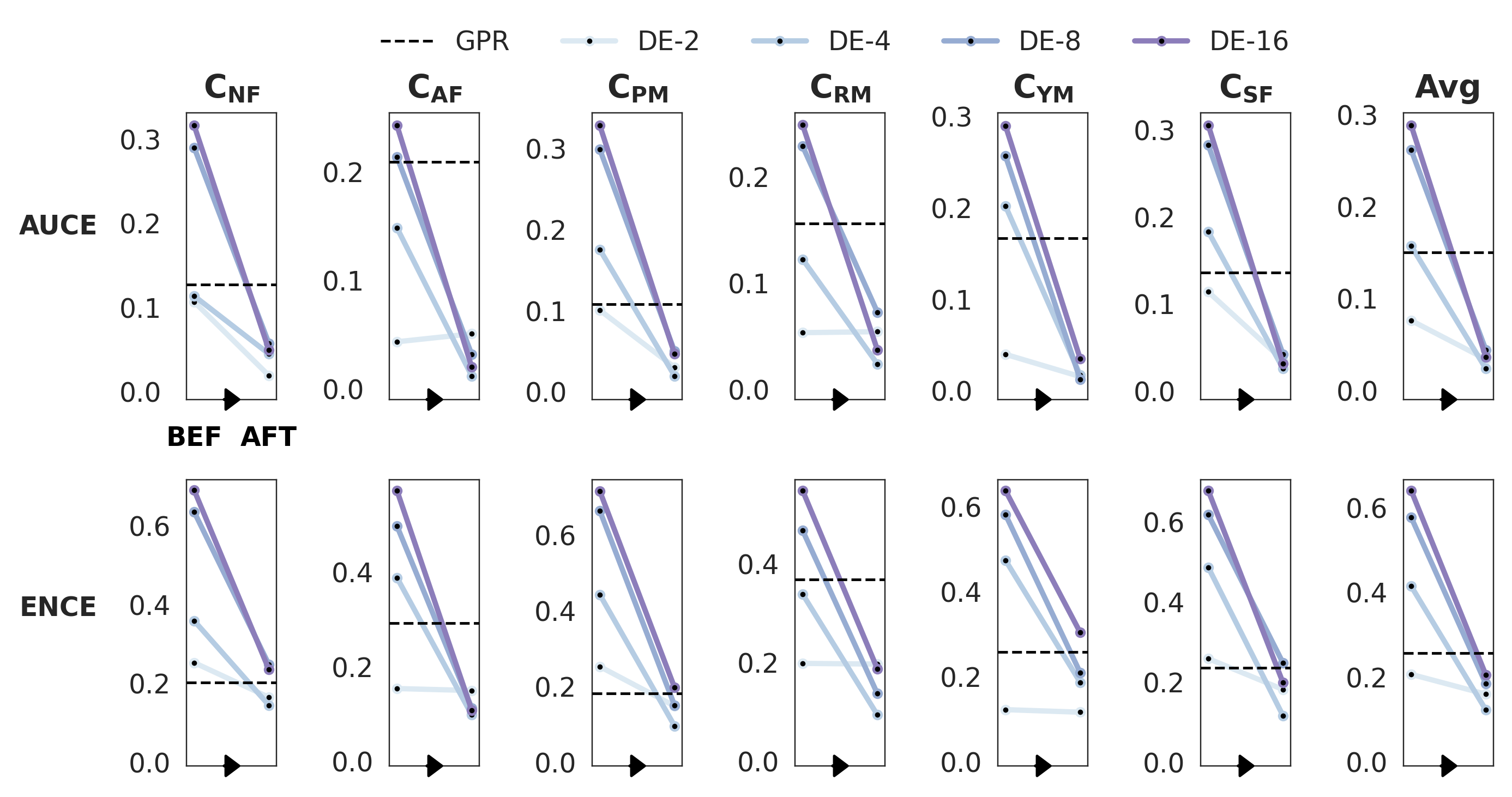}
    \caption{AUCE and ENCE of DE models before and after STD calibration. Those of GPR are also shown for comparison.}
    \label{fig:slope_plot}
\end{figure*} 

\clearpage
\subsection{\label{sec:UQ_BO} Effects of STD calibration on Exploratory Behavior in Bayesian optimization}

Since the scaling factors are optimized to have values less than 1 during the STD calibration process (Table \ref{tab:cal_results}), it is obvious that the overall predictive uncertainty of DE models would decrease. To provide a more intuitive understanding of the practical implications of calibration, this section aims to briefly point out that applying STD calibration to DE can lead to different exploratory behavior during Bayesian optimization. Specifically, the importance of calibration is highlighted by comparing the next query candidates before and after STD calibration obtained in the first iteration of Bayesian optimization. Note that only the first iteration is implemented in this paper for the following two reasons. First, the goal of this section is simply to show the impact of the calibration from a practical point of view. Second, the purpose of this section is not to claim that the final converged results of Bayesian optimization can be different depending on the calibration. Rather, it is to point out that the intended balance between exploitation and exploration may not be realized due to the miscalibrated uncertainty of the vanilla DE, which may affect the convergence behavior of the Bayesian optimization (by ``intended balance'' we mean the balance between exploration/exploitation in EI when the exact uncertainty is quantified).

Before moving on to Bayesian optimization, CIs of the 68\% confidence level predicted by DE-16 model are shown in Fig. \ref{fig:OOD} to visually understand the impact of calibration. Only one input variable, $AoA$, is used for the illustration. And its value is standardized to distinguish between its ID (in-distribution) region and the OOD (out-of-distribution) region: in Fig. \ref{fig:OOD}, the ID region is defined as the area containing 95$\%$ of the train data, while the OOD region is the remaining area. Overall, both results---those obtained before and after STD calibration---show diverging CIs in OOD and relatively narrow CIs in ID for all six QoIs. However, as expected from the scaling factors less than 1, the CIs from the DE-aft models become significantly narrower than those from the DE-bef models, indicating that these discrepancies will lead to differences in the subsequent Bayesian optimization process.

\begin{figure*}[htb!]
    \centering
        \includegraphics[width=0.8\textwidth]{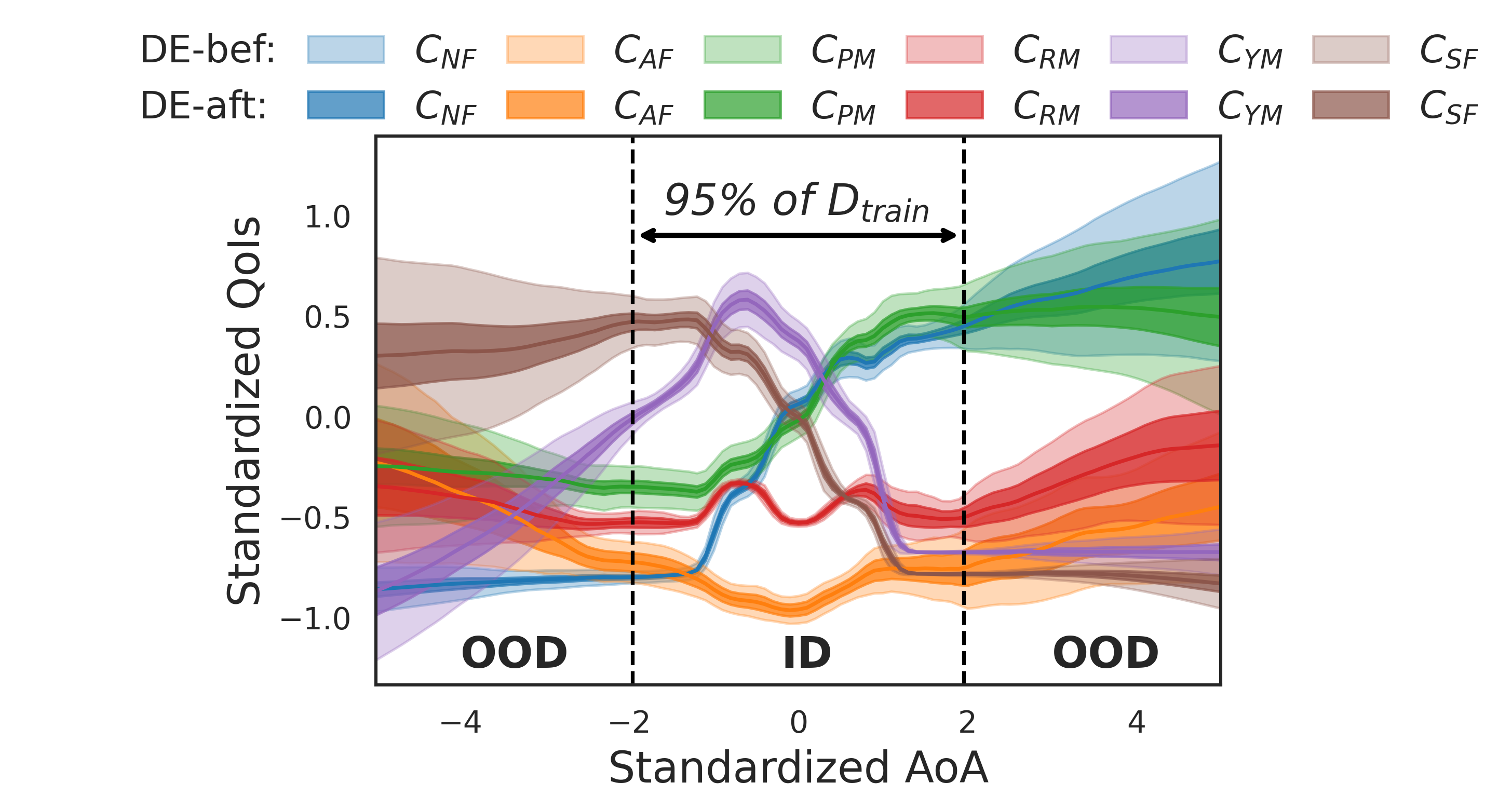}
    \caption{CIs of $68\%$ confidence level predicted by DE-16: comparison between before and after STD calibration.}
    \label{fig:OOD}
\end{figure*} 

Then, the multi-objective Bayesian optimization problem is defined is adopted to practically investigate their effects on Bayesian optimization: maximization of both $C_{NF}$ and $C_{AF}$ within five varying input parameters ($Ma$, $\phi$, $\delta{p}$, $\delta{r}$, and $AoA$). These optimizations, coupled with the expected improvement (EI) acquisition function, are performed separately for DE-bef and DE-aft models. The former searches for the maximum EI point where EI is calculated from the uncertainty quantified by the DE-bef model, while the latter does so using the uncertainty quantified by DE-aft. To find the Pareto solutions of $EI(C_{NF})$ and $EI(C_{AF})$, the non-dominated sorting genetic algorithm-\uppercase\expandafter{\romannumeral2} (NSGA-\uppercase\expandafter{\romannumeral2}) in the Python package pymoo is utilized \citep{blank2020pymoo, yang2022inverse, ozturk2006neuro}. Finally, the obtained Pareto solutions from the first iteration are shown in Fig. \ref{fig:EIopt_a}. Since the uncertainty estimated by DE-bef and DE-aft are different as shown in Fig. \ref{fig:OOD}, the Pareto solutions of $EI(C_{NF})$ and $EI(C_{AF})$ are also different: EI values of both QoIs after calibration are much smaller than those before calibration.

In Bayesian optimization, however, the most valuable information to the user is not the EI value itself (Fig. \ref{fig:EIopt_a}). More important are the values of the input variable sets (Fig. \ref{fig:EIopt_b}) obtained from the EI Pareto solutions: they are the next query candidates, the main purpose of implementing Bayesian optimization. Additional experiments/simulations will be performed on these candidates, indicating that their selection has a significant impact on the convergence behavior of the iterative Bayesian optimization process. If unintended candidates are obtained due to inaccurate UQ and therefore inaccurate EI calculation, the exploratory behavior of Bayesian optimization can be much different from the intention of the user. That is, the intended balance between exploitation and exploration during Bayesian optimization may differ due to unintentionally overestimated/underestimated uncertainty. To inspect its unintended exploratory behavior more intuitively, the parallel coordinates plot (PCP) in Fig. \ref{fig:EIopt_b} shows how the first query candidates in Bayesian optimization can vary due to the STD calibration in the DE model. This PCP has five vertical lines corresponding to each input variable, and the y-axis indicates their standardized values. Each red/blue line represents each point of the Pareto solutions in Fig. \ref{fig:EIopt_a}. Comparing them, large variations are found especially in the input variable $\delta{r}$: Bayesian optimization coupled with DE-bef discourages exploration of the variable $\delta{r}$ (which was not intended by the user), while DE-aft encourages exploration within $\delta{r}$ (which was the original intention). In conclusion, whether the DE is calibrated by STD calibration or not can result in exploration characteristics during Bayesian optimization that differ from the user's intent, which shows the impact of the calibration on DE from a practical point of view.

\begin{figure*}[htb!]
    \centering
    \begin{subfigure}[h]{0.45\textwidth}
        \centering
            \includegraphics[height=0.55\linewidth]{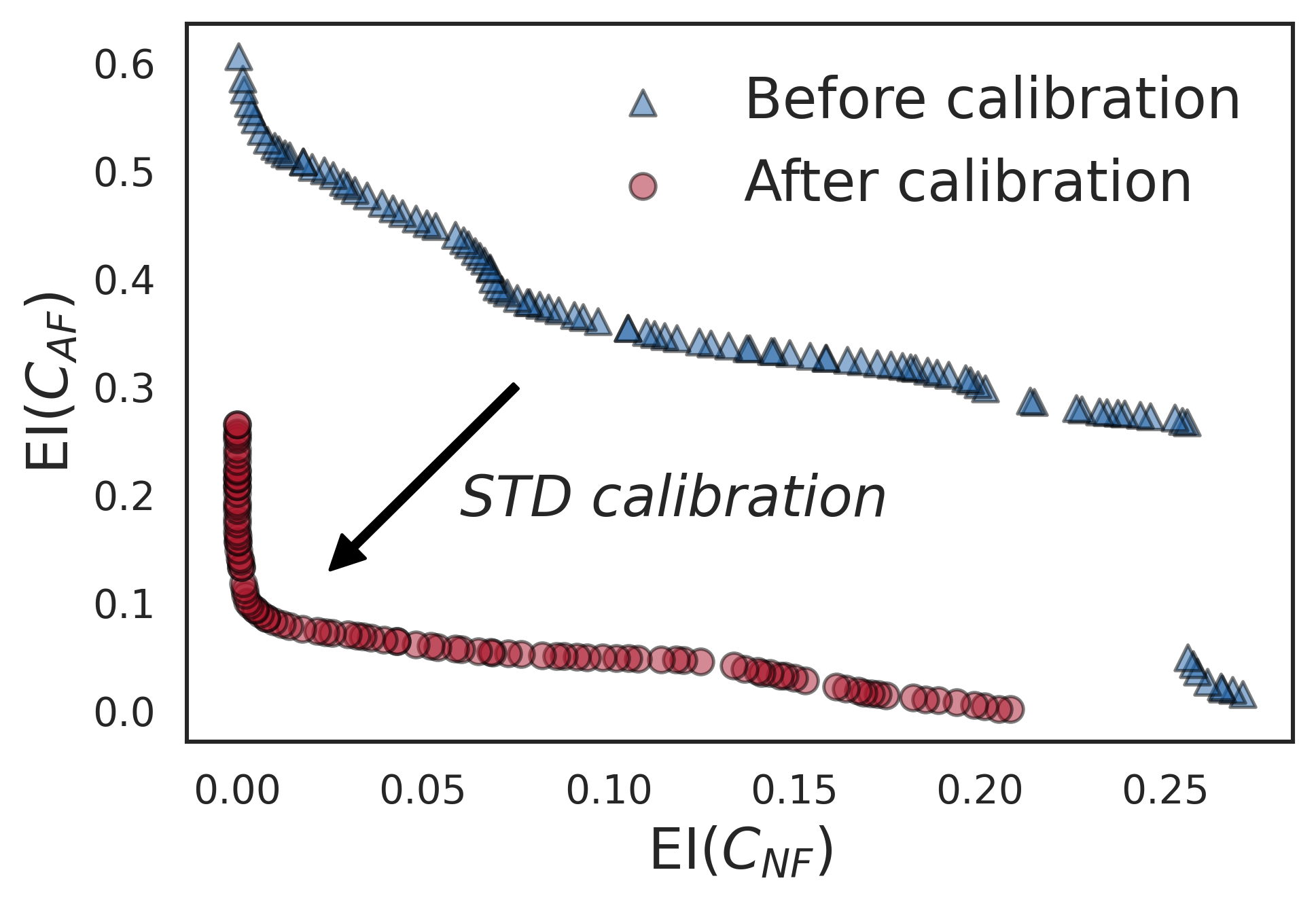}
        \subcaption{Pareto solutions obtained from multi-objective EI optimization}
        \label{fig:EIopt_a}
    \end{subfigure}% \qquad
    \hspace{0.05\columnwidth}
    \begin{subfigure}[h]{0.45\textwidth}
        \centering
            \includegraphics[height=0.55\linewidth]{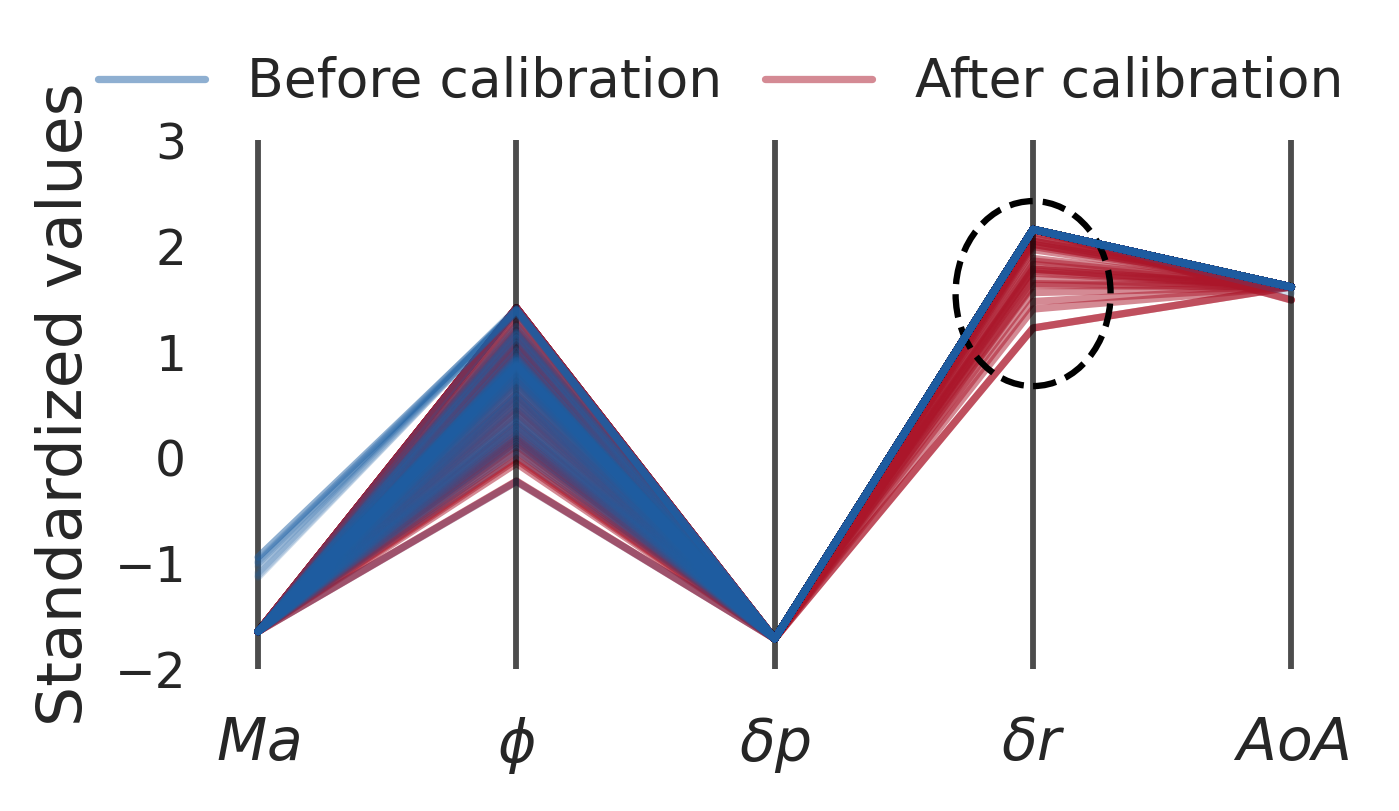}
        \subcaption{PCP of design variables in Pareto solutions}
        \label{fig:EIopt_b}
    \end{subfigure}
    \caption{Effects of STD calibration for DE models on Bayesian optimization results.}
    \label{fig:EIopt}
\end{figure*} 

\section{Conclusion}
\label{sec:conclusion}
This study comprehensively investigated the state-of-the-art approximate Bayesian inference approach, DE. It is applied to the multi-output regression task, which is the most common task in the engineering fields: a simple test case is adopted where aerodynamic QoIs of the specific missile configuration are predicted under varying flow conditions. DE models with different numbers of NNs are trained and then examined in the following order. First, their regression performance and the quality of estimated uncertainty are scrutinized while being compared with GPR. Then, a simple post-hoc STD calibration method is proposed to be applied to miscalibrated DE models. Finally, the effectiveness of the calibration on DE is highlighted by the improvement of two UQ quality criteria and the different exploratory behavior in Bayesian optimization before and after calibration. The key findings of our study can be summarized as follows:

\begin{enumerate}

    \item The effect of the number of NNs used in ensemble, $M$, is comprehensively investigated in the simple multi-output regression task. For regression accuracy, DE models show superior performance to GPR in terms of RMSE and NLL, while showing indistinguishable differences among themselves. For UQ quality, however, they show the obvious trend toward underconfidence as $M$ increases, both in terms of AUCE and ENCE criteria. The mathematical proof of why DE tends to be miscalibrated in regression tasks is also derived.
    
    \item The post-hoc STD calibration method, which simply modifies the estimated uncertainty from DE, is proposed to be applied to miscalibrated DE models. Finally, the reliability of the UQ performance after calibration is dramatically improved for both AUCE and ENCE, also surpassing that of GPR.
    
    \item The impact of the calibration approach on the exploratory behavior in Bayesian optimization is examined. Finally, whether or not the DE is calibrated via STD calibration can result in completely different exploration characteristics when extended to Bayesian optimization, which cautions against blindly applying vanilla DE models to Bayesian optimization in regression tasks.
    
    \item We have demonstrated that by applying a simple post-hoc STD calibration technique that requires negligible additional post-processing time, DE models can have enormous potential compared to GPR, which is the most commonly used regression model for UQ in engineering. These results are summarized in Table \ref{tab:app_public}, where the DE-2 model after STD calibration outperforms GPR in terms of regression performance ($-56\%$ NLL $\&$ $-55\%$ RMSE), reliability of UQ ($-77\%$ AUCE $\&$ $-38\%$ ENCE), and training efficiency ($-78\%$ training time).
    
\end{enumerate}

% https://tex.stackexchange.com/questions/167366/combining-multirow-and-multicolumn
\renewcommand{\arraystretch}{1.3}
\begin{table}[htb!]
\caption{Comprehensive comparison between GPR and DE-2}\label{tab:app_public}
    \begin{NiceTabular*}{0.8\columnwidth}{@{\extracolsep{\fill}}lc!{\qquad}ccc}
        \cline{1-5}
        \Block[c]{2-2}{Metrics} && \Block{2-1}{GPR} & \Block{1-2}{DE-2} & \\ \cline{4-5}
        & & & Before calibration & After calibration \\ \cline{1-5}
        \Block{2-1}{Regression} & NLL & -2.653 ($-\%$) & -4.145 ($\downarrow\textbf{56}\%$) & -4.145 ($\downarrow\textbf{56}\%$)\\ \cline{2-5}
        & RMSE & 0.029 ($-\%$) & 0.013 ($\downarrow\textbf{55}\%$) & 0.013 ($\downarrow\textbf{55}\%$) \\ \cline{1-5}
        \Block{2-1}{UQ} & AUCE & 0.150 ($-\%$) & 0.076 ($\downarrow49\%$) & 0.034 ($\downarrow\textbf{77}\%$)\\ \cline{2-5}
        & ENCE & 0.256 ($-\%$) & 0.206 ($\downarrow20\%$) & 0.159 ($\downarrow\textbf{38}\%$)\\ \cline{1-5}
        \Block[c]{1-2}{Training time [s]} && 39081 ($-\%$) & 8640 ($\downarrow\textbf{78}\%$) & 8640+30 ($\downarrow\textbf{78}\%$) \\ \cline{1-5}                
\end{NiceTabular*}
\end{table}

The presented DE framework has great promise in two engineering applications. First, DE with STD calibration has the potential to replace the most common regression model, GPR, owing to its following advantages: more scalable to large datasets, higher regression accuracy, and last but not least, more reliable uncertainty estimation. Second, DE with STD calibration can be leveraged in Bayesian optimization by ensuring a reliable balance between exploitation and exploration due to its trustworthy UQ performance. Although the application of this framework has been demystified using the simple multi-output regression task, it can be easily applied and extended to high-dimensional input/output problems since it is based on the deep neural network structures and no special assumptions have been made for this specific problem. For future work, a more practical investigation of DE models will be conducted, such as their scalability to high-dimensional engineering regression problems, not as simple regression task with 6 outputs in this study. Also, since the DE algorithm by its nature requires multiple network training, the way to reduce its longer training time than the conventional NNs can also be a future study. Finally, since our purpose was not to claim that the final converged results of Bayesian optimization can be different depending on the calibration, the extension of DE to the whole Bayesian optimization was outside the focus of our paper. In this respect, a comprehensive comparison of the convergence behavior between DE-bef and DE-aft over entire iterations and their final converged results can be invaluable future work.

\clearpage
\section*{CRediT authorship contribution statement}
\textbf{S. Yang}: Conceptualization, Methodology, Software, Validation, Formal analysis, Investigation, Data Curation, Writing – Original Draft, Writing – Review \& Editing, Visualization.
\textbf{K. Yee}: Supervision, Funding acquisition.

\section*{Declaration of competing interest}
The authors declare that they have no known competing financial interests or personal relationships that could have appeared to influence the work reported in this paper.

\section*{Data availability}
Data will be made available on request.

\section*{Acknowledgments}
This work was supported by Grant UE191109CD from Agency for Defense Development and Defense Acquisition Program Administration \& Hyundai Motor Group. Also, the authors especially thank Shinkyu Jeong and Seungmin Yoo at Kyunghee University for providing the training dataset of the current study.

\clearpage
\appendix

\section{Controversial issues on MC-dropout}
\label{sec:app_MCD}

\citet{osband2016risk} pointed out that what MC-dropout (MCD) estimates is a risk, not an uncertainty, and also emphasized the pitfalls of MCD when used as a naive tool for estimating uncertainty. Moreover, its algorithm does not perform adequately even in very simple examples \citep{osband2016deep, pearce2018bayesian}, and its posterior samples are often too spiky to provide a reliable predictive uncertainty trend \citep{gal2016dropout, gal2016uncertainty, osband2016deep, riquelme2018deep, zhang2019quantifying}, which makes it unattractive to be exploited for Bayesian optimization.

\section{Hyperparameter tuning results in Section \ref{sec:pre_missile}}
\label{sec:app_hyp}

\subsection{Results of DE models}
\label{sec:app_hyp_DE}

The results of the hyperparameter tuning for DE models performed in Section \ref{sec:pre_missile} are shown in Table \ref{tab:hyp_results}. Three hyperparameters are used: the number of hidden layers ($N_{layer}\in\{3, 5, 7\}$), the number of nodes in each hidden layer ($N_{node}\in\{32, 64, 128\}$), and the size of the mini-batch ($N_{batch}\in\{512, 1024, 2048\}$). The corresponding regression performance in terms of NLL and RMSE for all hyperparameter combinations is shown (total training time of 6944 seconds). Since $N_{layer}=7, N_{node}=128, N_{batch}=512$ shows the best RMSE performance, it is selected as the best hyperparameter combination.

\renewcommand{\arraystretch}{1.1}
\begin{table}[H]
\caption{Results of hyperparameter tuning: several structures of probabilistic NN used in the DE model are tested.}\label{tab:hyp_results}
    \begin{tabular*}{0.45\columnwidth}{@{\extracolsep{\fill}}cccccc}
        \cline{1-6}
        \begin{tabular}[c]{@{}c@{}}$N_{layer}$\end{tabular} & \begin{tabular}[c]{@{}c@{}}$N_{node}$\end{tabular} & \begin{tabular}{@{}c@{}}$N_{batch}$\end{tabular} & \begin{tabular}[c]{@{}c@{}}NLL\end{tabular} & \begin{tabular}{@{}c@{}}RMSE\end{tabular} & \begin{tabular}{@{}c@{}}Time [s]\end{tabular}\\ \cline{1-6}
        \multirow{9}{*}{3} & \multirow{3}{*}{32} & 512 & -2.50 & 0.125 & 183 \\ \cline{3-6}
         &  & 1024 & -2.22 & 0.200 & \textbf{174}\\ \cline{3-6}
         &  & 2048 & -2.13 & 0.183 & 183\\ \cline{2-6}
         & \multirow{3}{*}{64} & 512 & -2.98 & 0.076 & 202 \\ \cline{3-6}
         &  & 1024 & -2.69 & 0.087 & 179\\ \cline{3-6}
         &  & 2048 & -2.52 & 0.100 & 213\\ \cline{2-6}
         & \multirow{3}{*}{128} & 512 & -3.37 & 0.036 & 256 \\ \cline{3-6}
         &  & 1024 & -3.34 & 0.040 & 257\\ \cline{3-6}
         &  & 2048 & -2.55 & 0.052 & 271\\ \cline{1-6}
        \multirow{9}{*}{5} & \multirow{3}{*}{32} & 512 & -2.62 & 0.137 & 187\\ \cline{3-6}
         &  & 1024 & -2.23 & 0.173 & 214\\ \cline{3-6}
         &  & 2048 & -2.25 & 0.137 & 202\\ \cline{2-6}
         & \multirow{3}{*}{64} & 512 & -3.21 & 0.039 & 236\\ \cline{3-6}
         &  & 1024 & -2.79 & 0.055 & 210\\ \cline{3-6}
         &  & 2048 & -2.01 & 0.085 & 255\\ \cline{2-6}
         & \multirow{3}{*}{128} & 512 & \textbf{-3.83} & 0.017 & 342\\ \cline{3-6}
         &  & 1024 & -3.53 & 0.019 & 327\\ \cline{3-6}
         &  & 2048 & -3.16 & 0.035 & 349\\ \cline{1-6}
        \multirow{9}{*}{\textbf{7}} & \multirow{3}{*}{32} & 512 & -2.49 & 0.088 & 212\\ \cline{3-6}
         &  & 1024 & -2.36 & 0.119 & 224\\ \cline{3-6}
         &  & 2048 & -2.17 & 0.104 & 244\\ \cline{2-6}
         & \multirow{3}{*}{64} & 512 & -3.38 & 0.031 & 286\\ \cline{3-6}
         &  & 1024 & -3.05 & 0.042 & 247\\ \cline{3-6}
         &  & 2048 & -2.45 & 0.062 & 262\\ \cline{2-6}
         & \multirow{3}{*}{\textbf{128}} & \textbf{512} & -3.82 & \textbf{0.016} & 428\\ \cline{3-6}
         &  & 1024 & -3.56 & 0.024 & 399\\ \cline{3-6}
         &  & 2048 & -3.13 & 0.035 & 402\\ \cline{1-6}
    \end{tabular*}
\end{table}

\clearpage

\subsection{Results of GPR models}
\label{sec:app_hyp_GPR}

The results of the hyperparameter tuning for GPR models performed in Section \ref{sec:pre_missile} are shown in Table \ref{tab:hyp_results_GPR}. For conventional single-output GPR (SOGPR), Mat\'ern 5/2, radial basis function, rational quadratic, and dot-product kernels are explored. Additionally, multi-output GPR (MOGPR) with radial basis function is also tested. The corresponding regression performance in terms of NLL and RMSE of all GPR models are summarized in Table \ref{tab:hyp_results_GPR} (total training time is 205493 seconds, which is significantly longer than DE in \ref{sec:app_hyp_DE}). MOGPR requires the least training time, but single-output GPR with Mat\'ern 5/2 is selected since it shows the best performance with respect to NLL and RMSE.

\begin{table}[htb!]
\caption{Results of hyperparameter tuning: several GPR models are tested.}\label{tab:hyp_results_GPR}
    \begin{tabular*}{0.7\columnwidth}{@{\extracolsep{\fill}}cccc}
    \hline
    Kernels & NLL & RMSE & Time [s] \\ \hline
    \textbf{Mat\'ern 5/2} & \textbf{-2.653} & \textbf{0.029} & 39081\\
    Radial basis function (SOGPR) & -1.657 & 0.039 & 64250\\
    Radial basis function (MOGPR) & -1.4236 & 0.044 & \textbf{8136}\\
    Rational quadratic & -0.747 & 0.061 & 66363\\
    Dot-product & 0.440 & 0.547 & 27663\\ \hline
    \end{tabular*}
\end{table}

\clearpage

\section{Additional results in Section \ref{sec:UQ_results}}
\label{sec:app_UQ_bef}

Fig. \ref{fig:reliab_DE_bef} in Section \ref{sec:UQ_results} shows the reliability plots with respect to only one QoI, $C_{SF}$. In this section, more comprehensive results are provided. For each vanilla DE model, all six QoIs are shown with different colors. Again, the trend of underconfidence as $M$ increases can be seen from DE-2 to DE-16.

\begin{figure*}[ht!]
    \centering
    \begin{subfigure}[h]{0.8\textwidth}
        \centering
        \includegraphics[height=0.18\textheight]{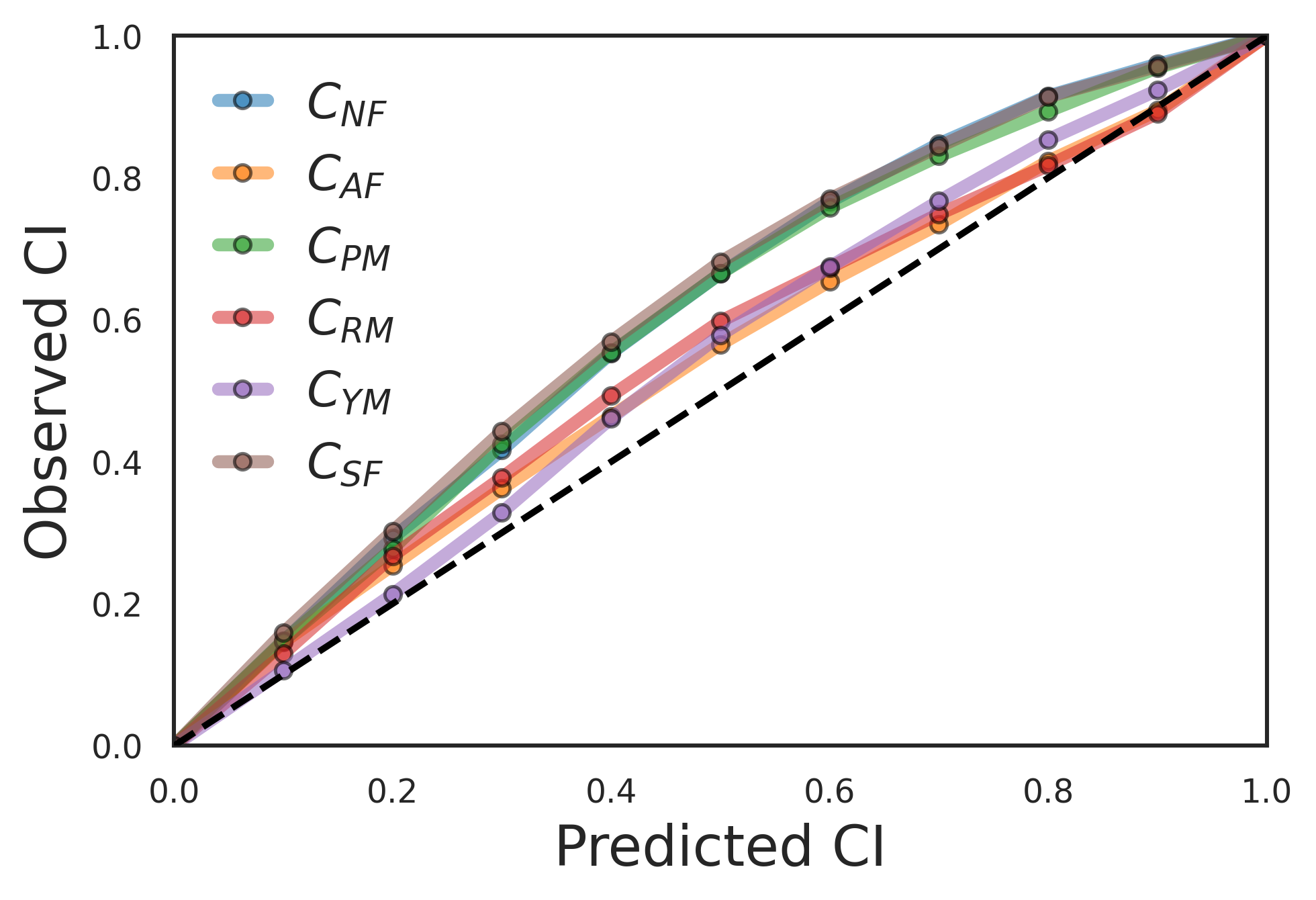}
        \hspace{0.0\columnwidth}
        \includegraphics[height=0.18\textheight]{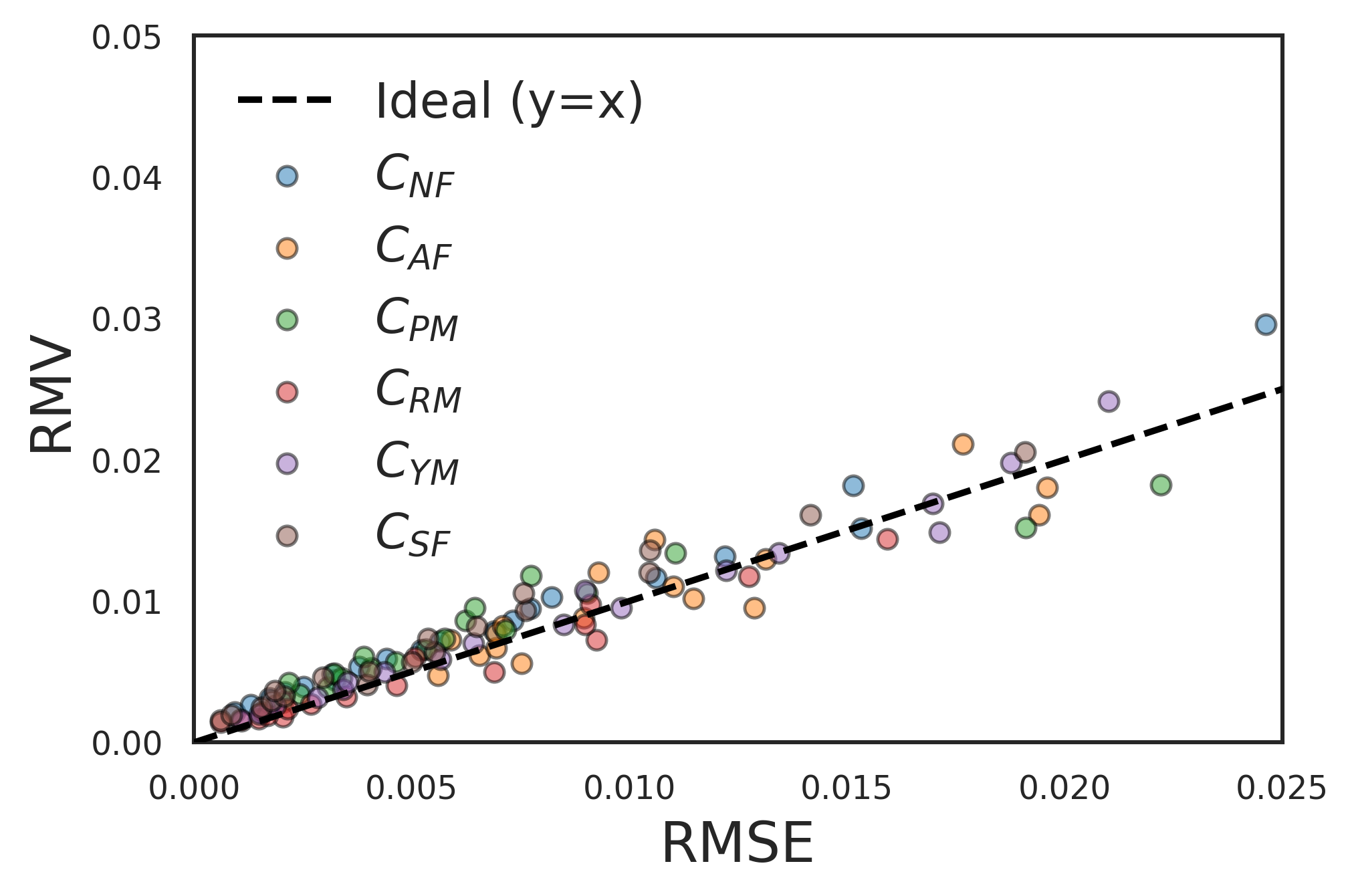}
        \caption{DE-2}\label{fig:comp_NN_a}
    \end{subfigure}
    \vfill
    \begin{subfigure}[h]{0.8\textwidth}
        \centering
        \includegraphics[height=0.18\textheight]{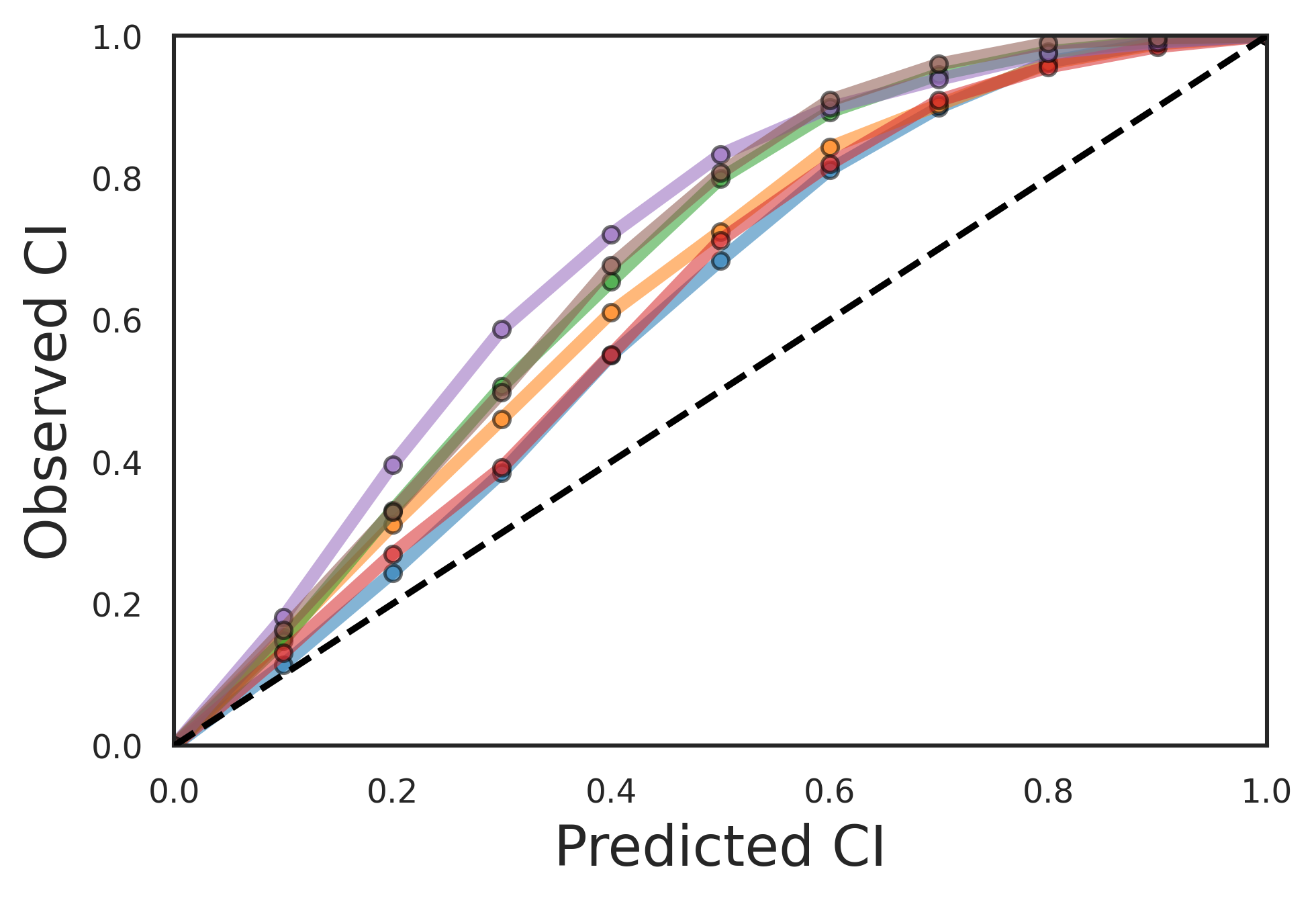}
        \hspace{0.0\columnwidth}
        \includegraphics[height=0.18\textheight]{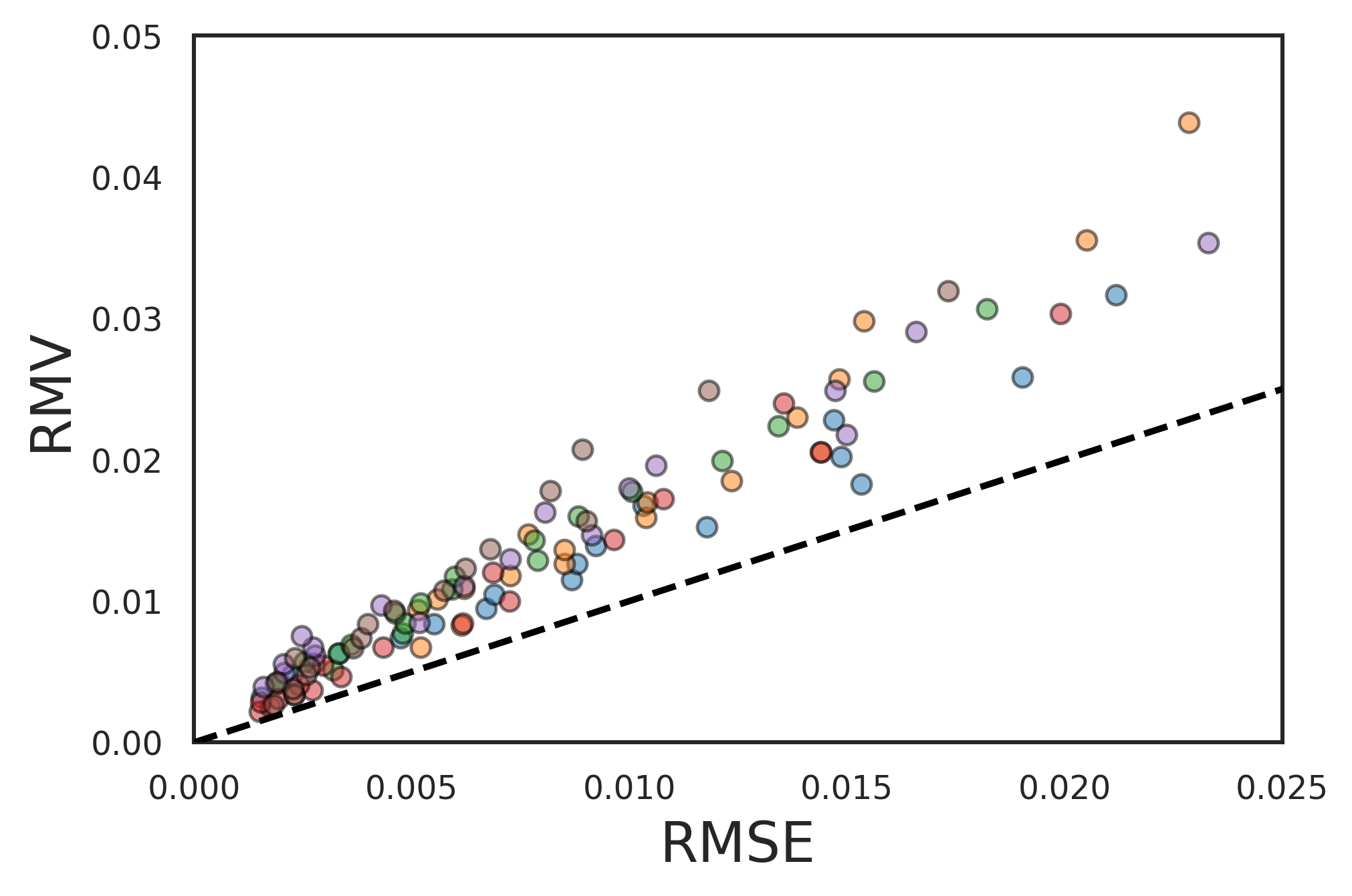}
        \caption{DE-4}\label{fig:comp_NN_b}
    \end{subfigure}
    \vfill
    \begin{subfigure}[h]{0.8\textwidth}
        \centering
        \includegraphics[height=0.18\textheight]{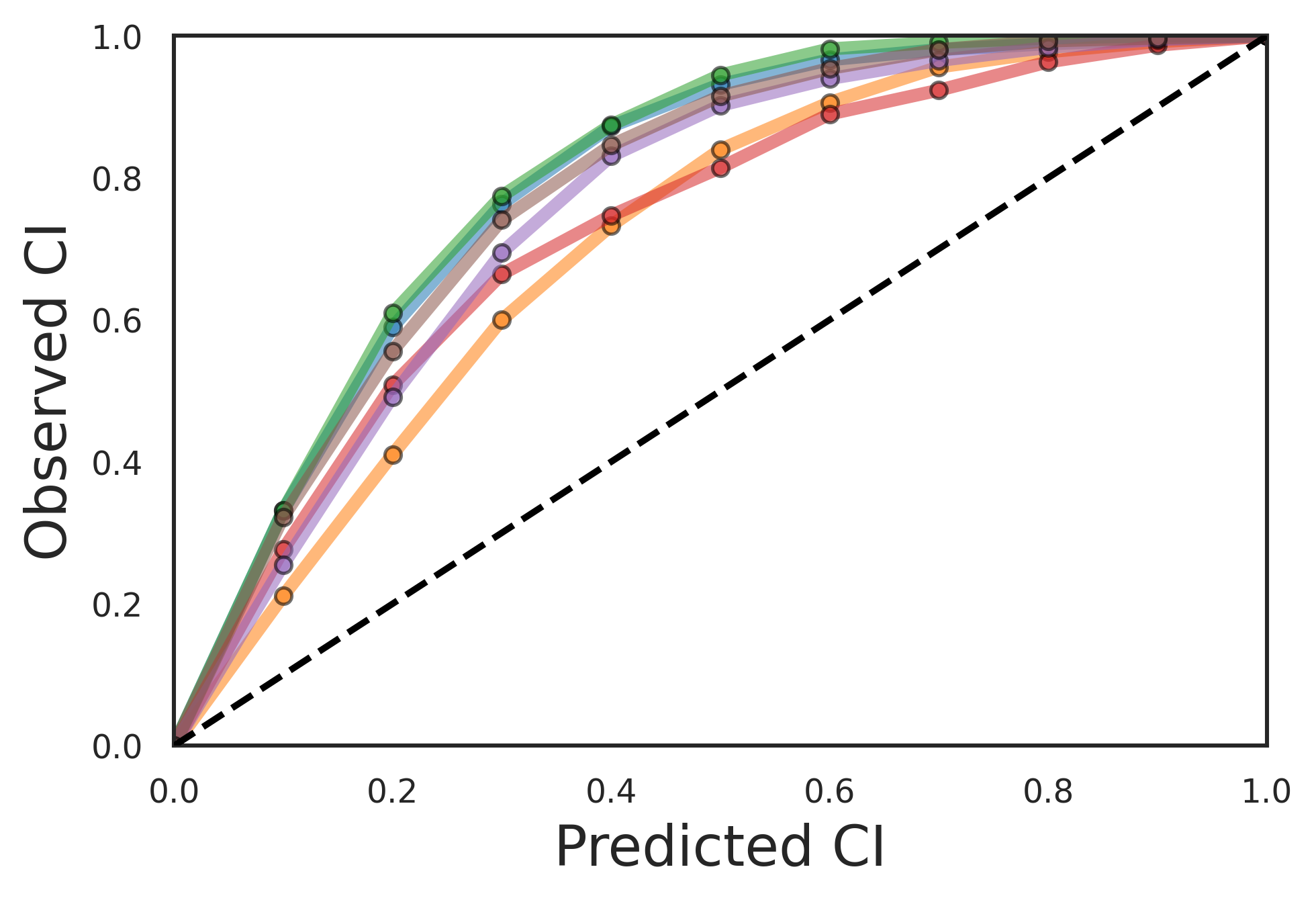}
        \hspace{0.0\columnwidth}
        \includegraphics[height=0.18\textheight]{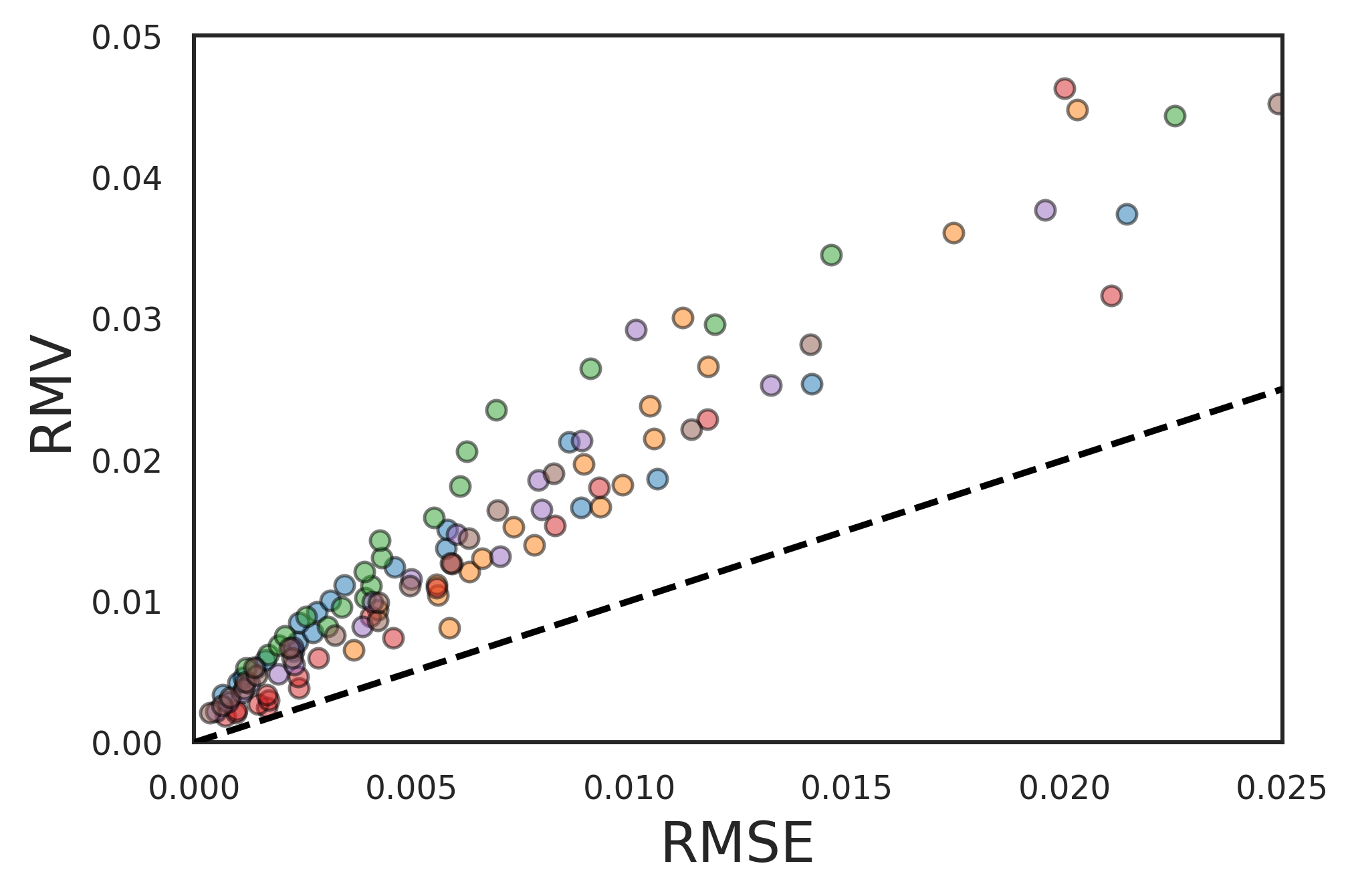}
        \caption{DE-8}\label{fig:comp_NN_c}
    \end{subfigure}
    \vfill
    \begin{subfigure}[h]{0.8\textwidth}
        \centering
        \includegraphics[height=0.18\textheight]{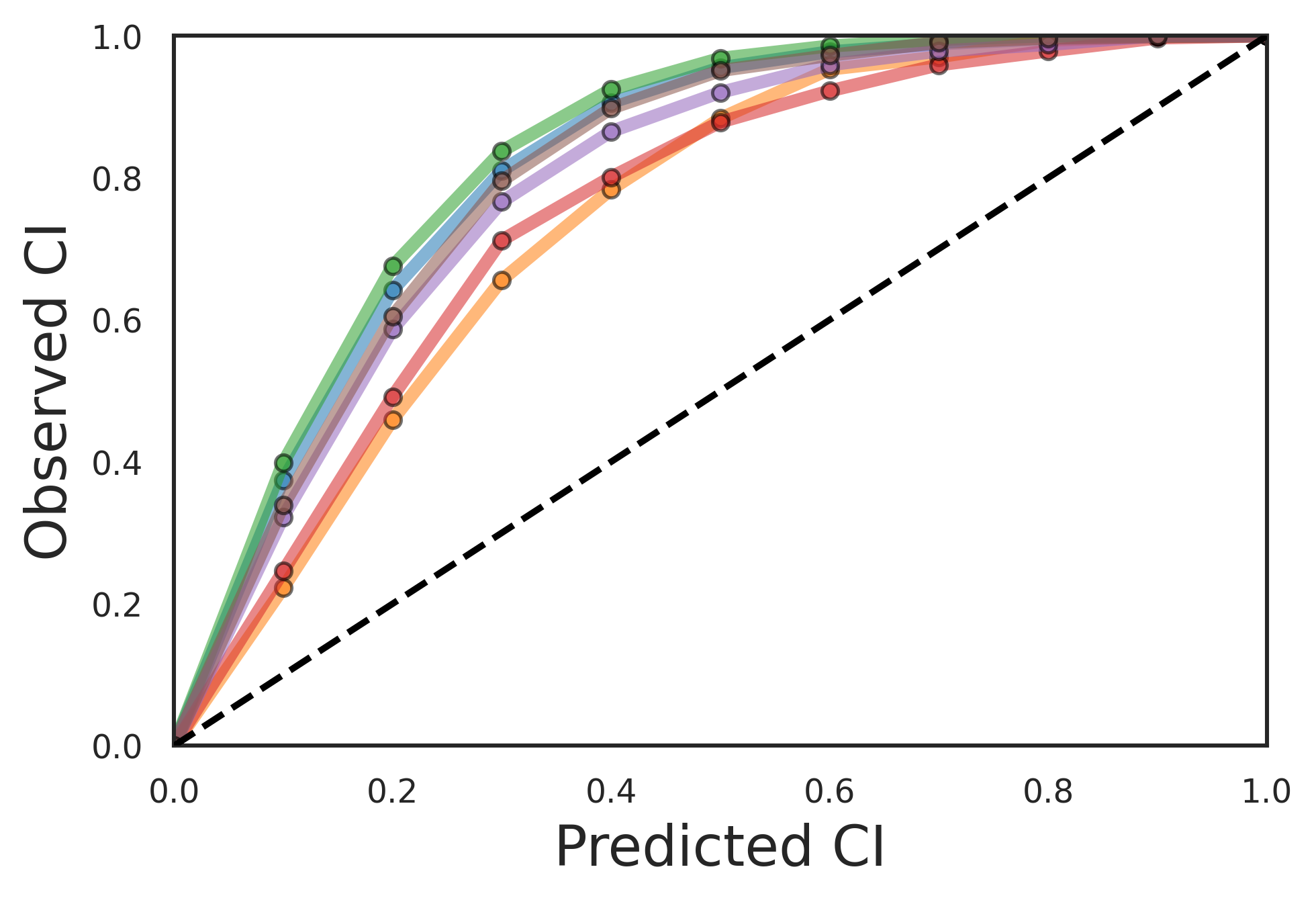}
        \hspace{0.0\columnwidth}
        \includegraphics[height=0.18\textheight]{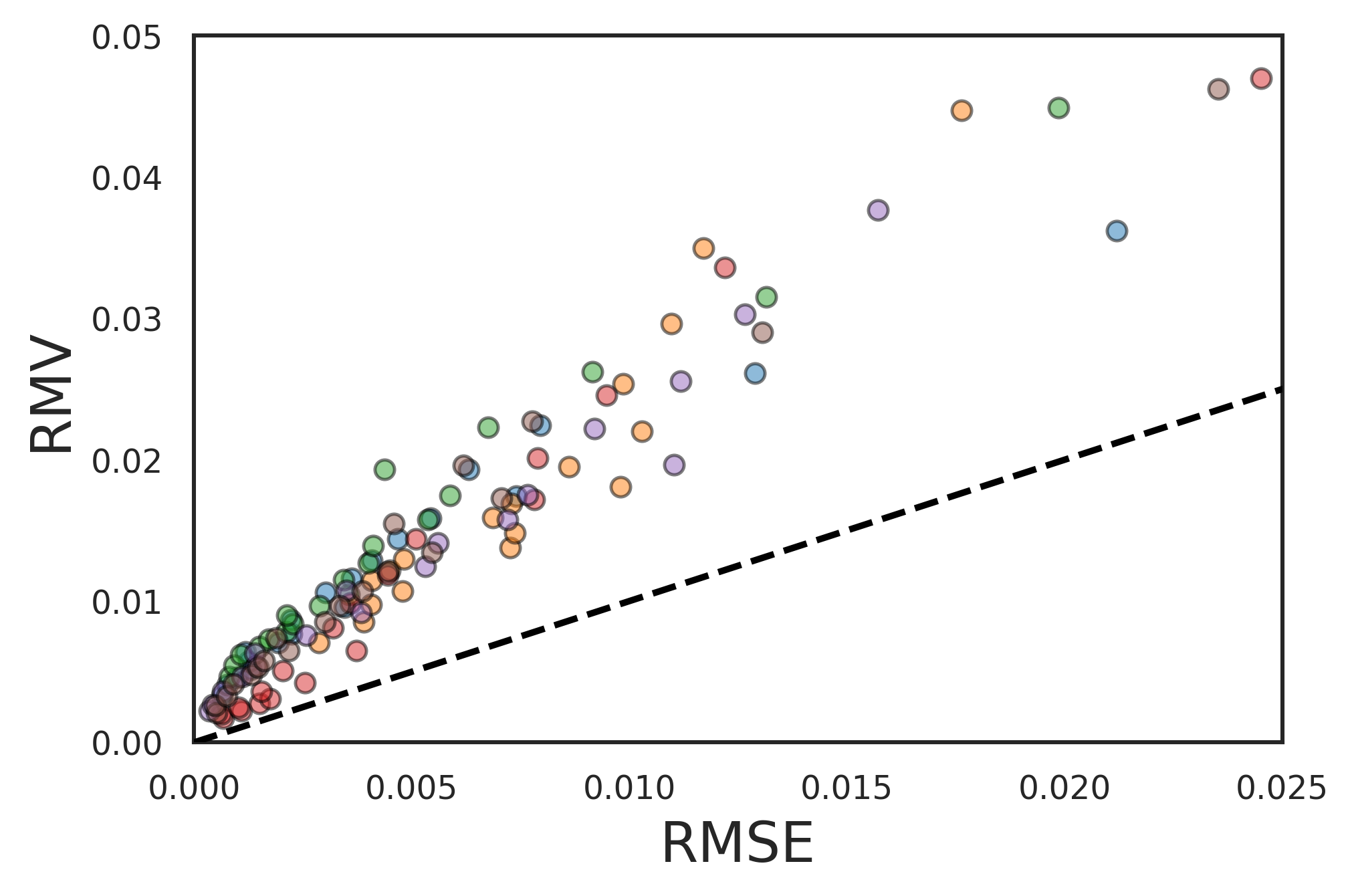}
        \caption{DE-16}\label{fig:comp_NN_d}
    \end{subfigure}
    \caption{Reliability plots of vanilla DE models: (left) CI-based reliability plots, (right) error-based reliability plots.}
    \label{fig:reliab_DE_comprehensive}
\end{figure*}

\clearpage
\section{Additional results in Section \ref{sec:UQ_calib}}
\label{sec:app_UQ_aft}

In \ref{sec:app_UQ_bef}, the reliability plots of DE models before STD calibration (vanilla DE models) are shown; this section shows the results after STD calibration. It is shown that all DE models become well-calibrated after calibration, even for the DE-16 model, which was the most miscalibrated DE model.

\begin{figure*}[ht!]
    \centering
    \begin{subfigure}[h]{0.8\textwidth}
        \centering
        \includegraphics[height=0.18\textheight]{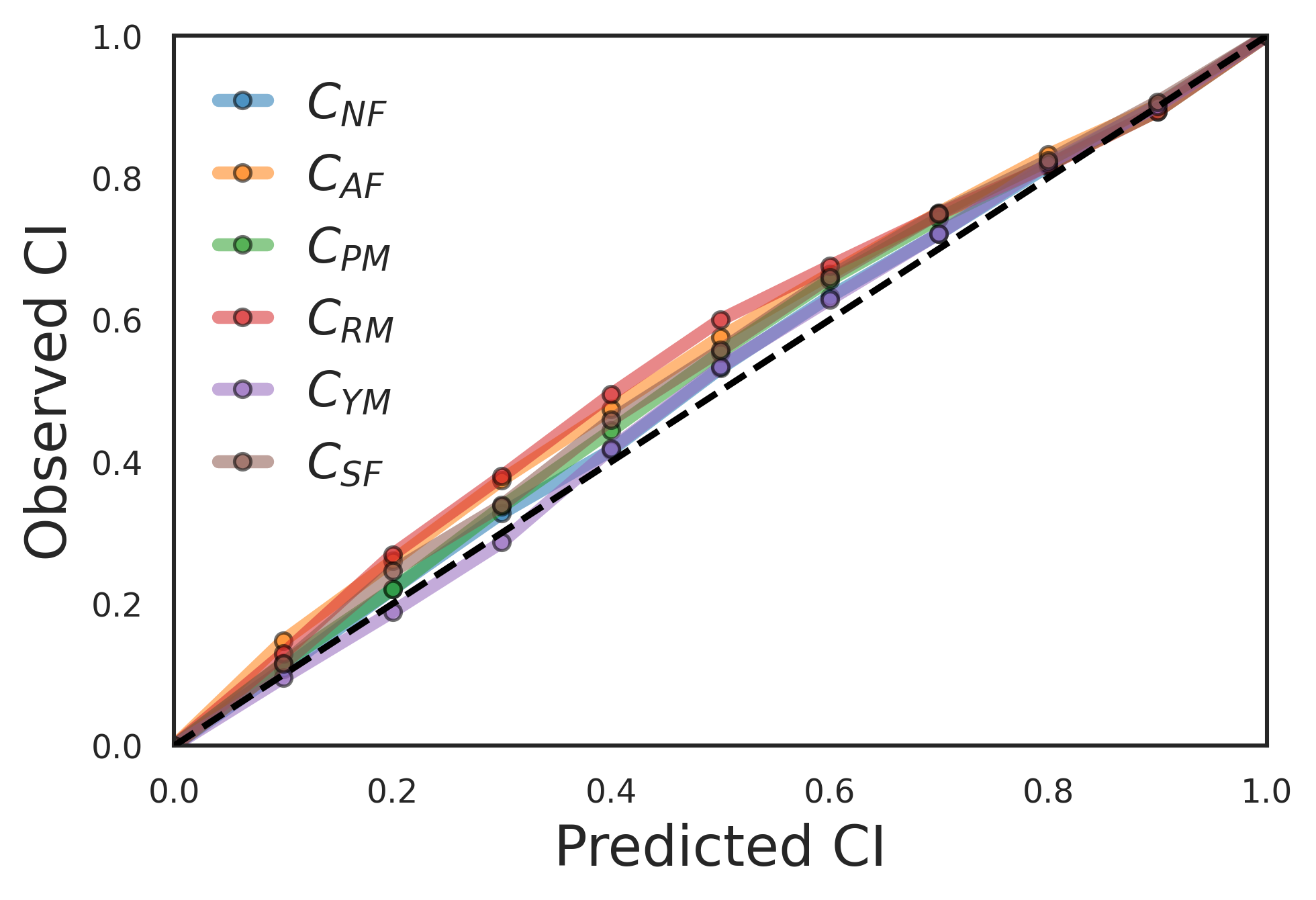}
        \hspace{0.0\columnwidth}
        \includegraphics[height=0.18\textheight]{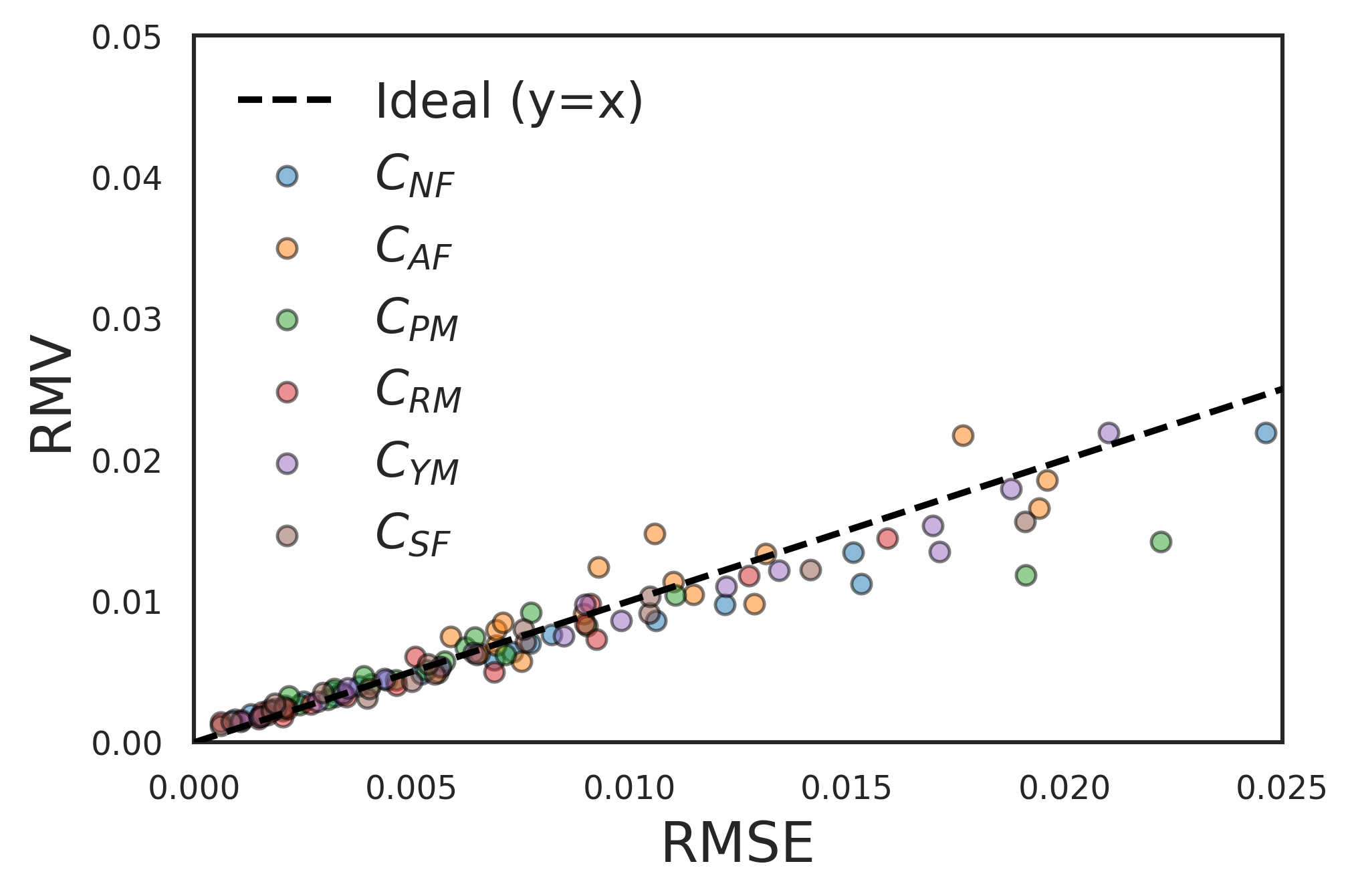}
        \caption{DE-2}\label{fig:calcomp_NN_a}
    \end{subfigure}
    \vfill
    \begin{subfigure}[h]{0.8\textwidth}
        \centering
        \includegraphics[height=0.18\textheight]{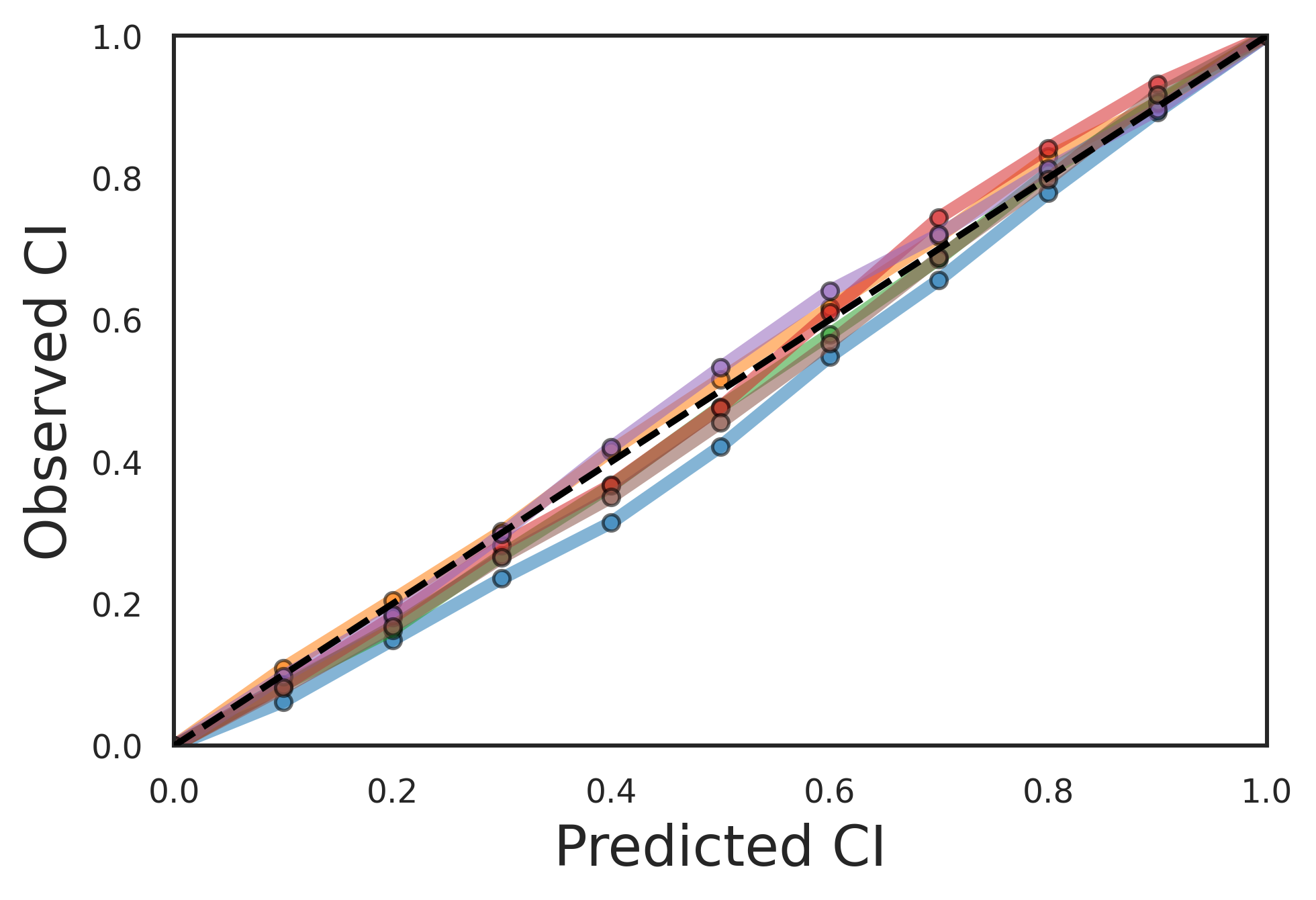}
        \hspace{0.0\columnwidth}
        \includegraphics[height=0.18\textheight]{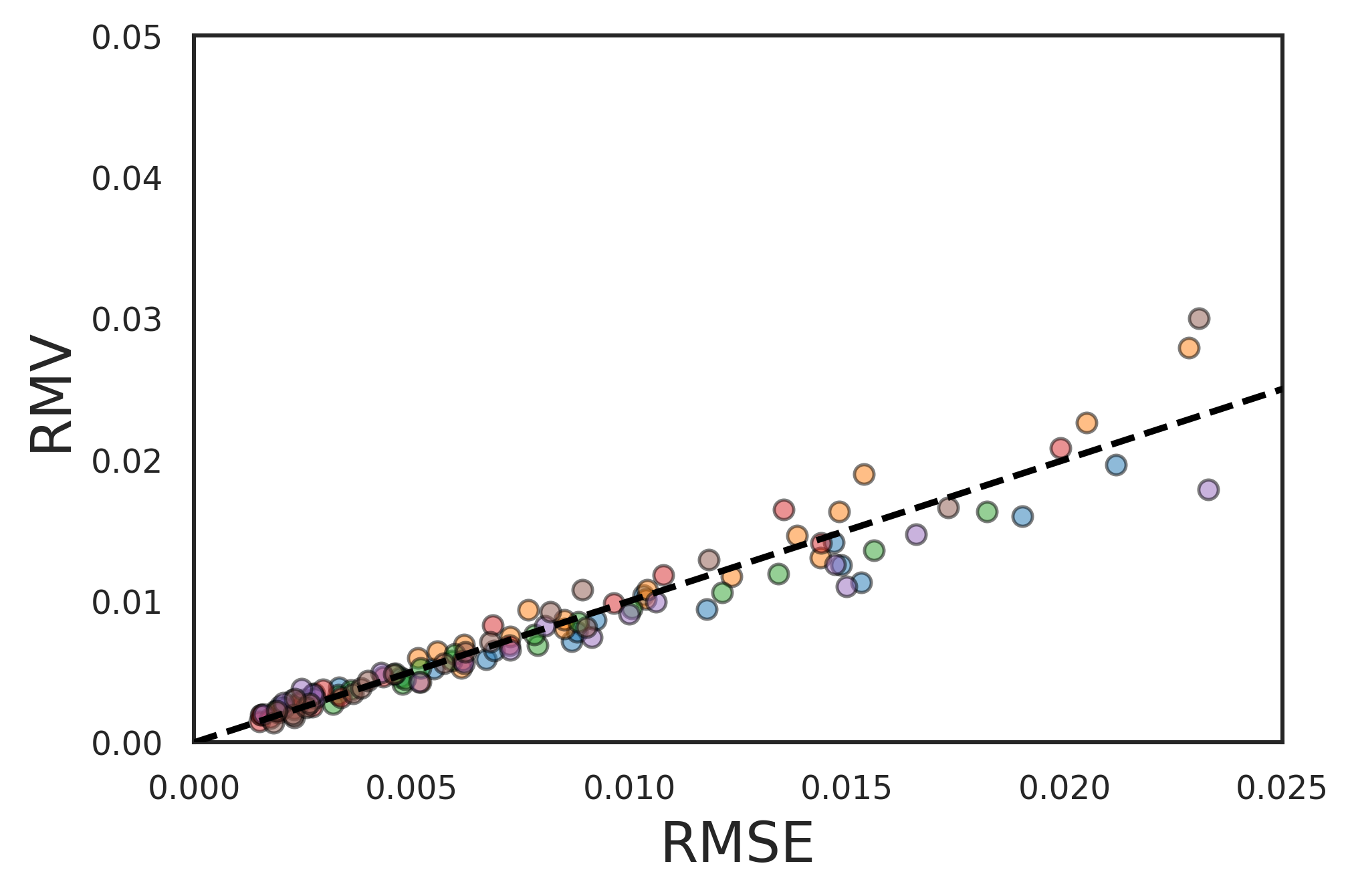}
        \caption{DE-4}\label{fig:calcomp_NN_b}
    \end{subfigure}
    \vfill
    \begin{subfigure}[h]{0.8\textwidth}
        \centering
        \includegraphics[height=0.18\textheight]{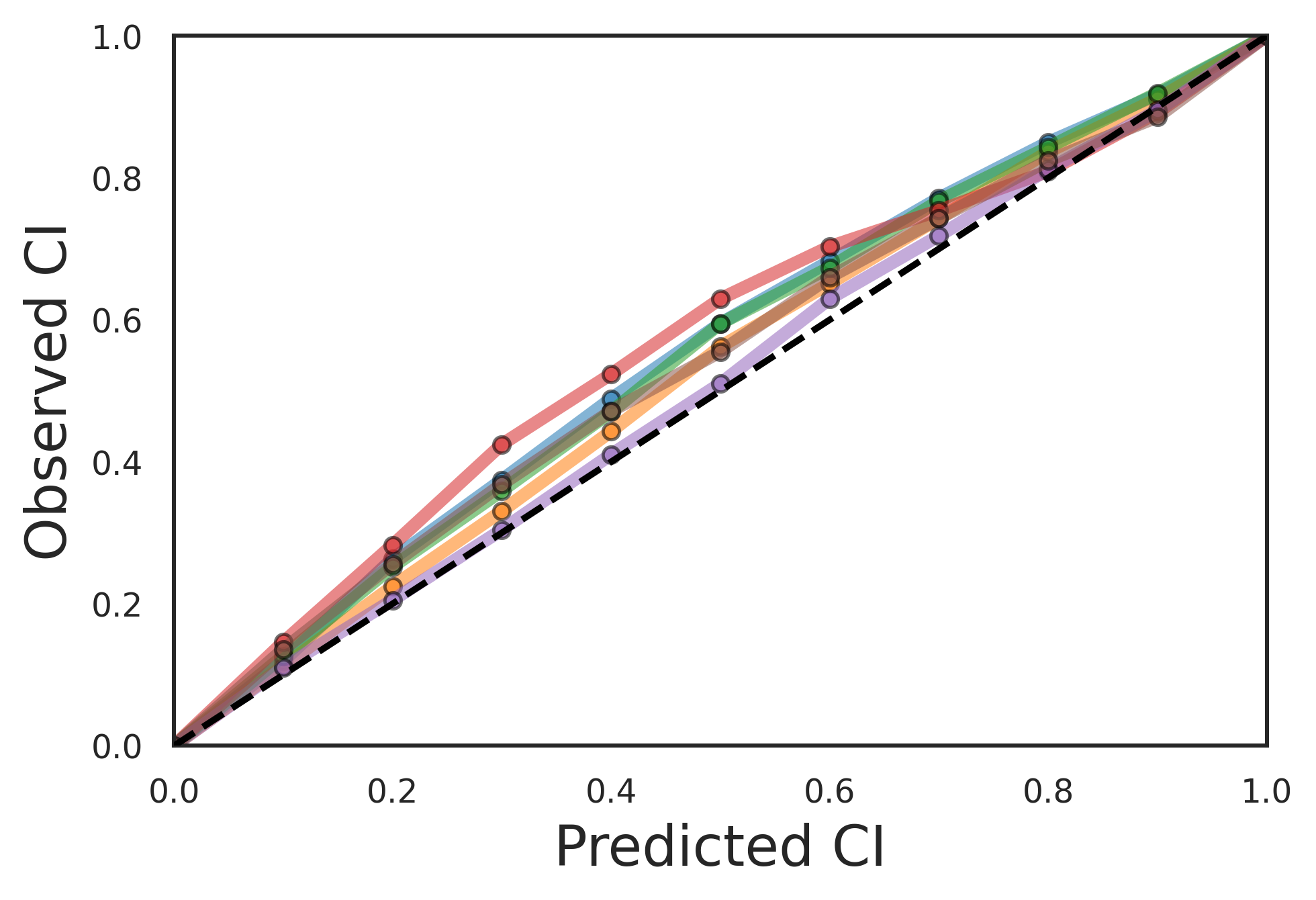}
        \hspace{0.0\columnwidth}
        \includegraphics[height=0.18\textheight]{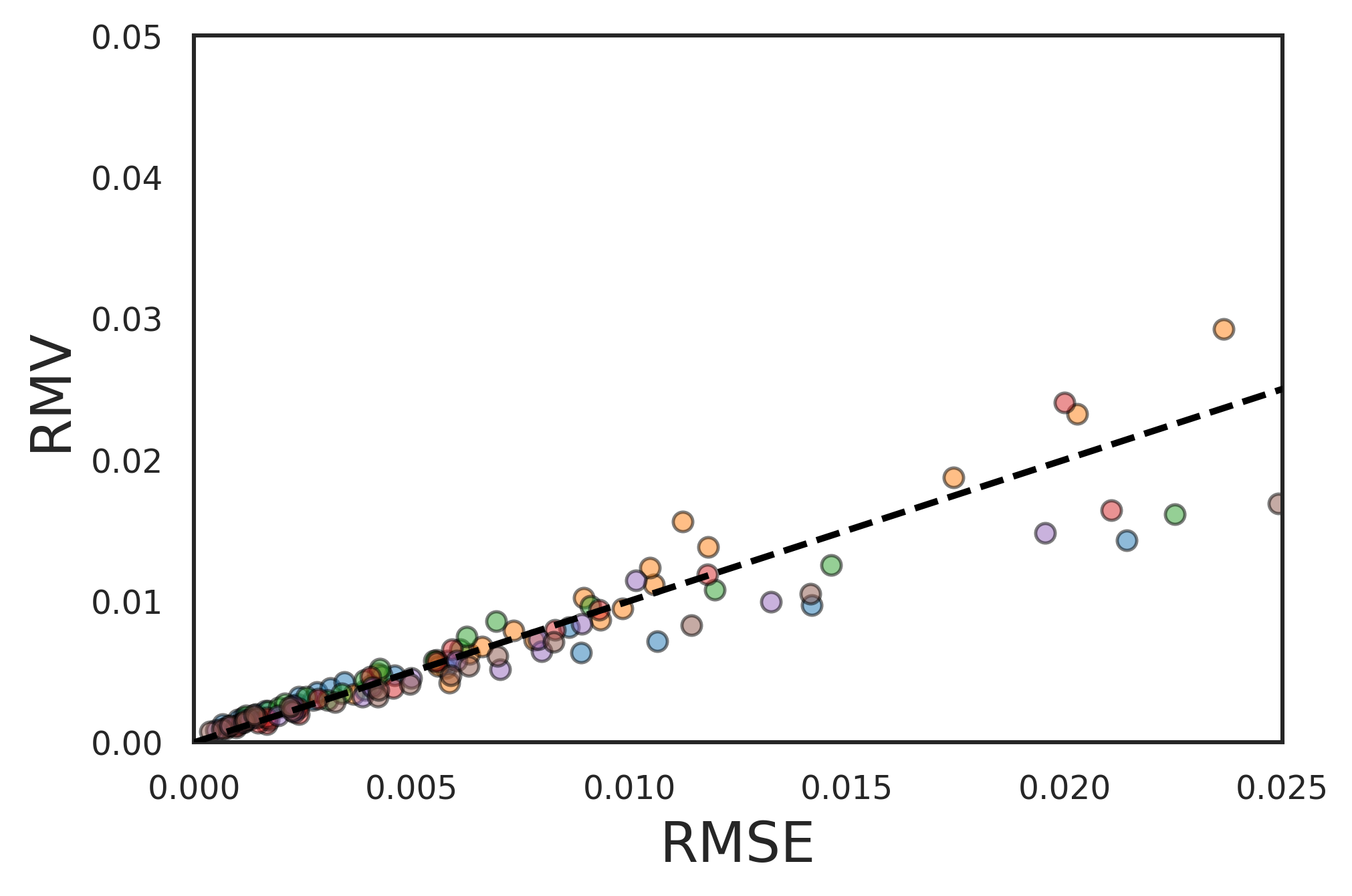}
        \caption{DE-8}\label{fig:calcomp_NN_c}
    \end{subfigure}
    \vfill
    \begin{subfigure}[h]{0.8\textwidth}
        \centering
        \includegraphics[height=0.18\textheight]{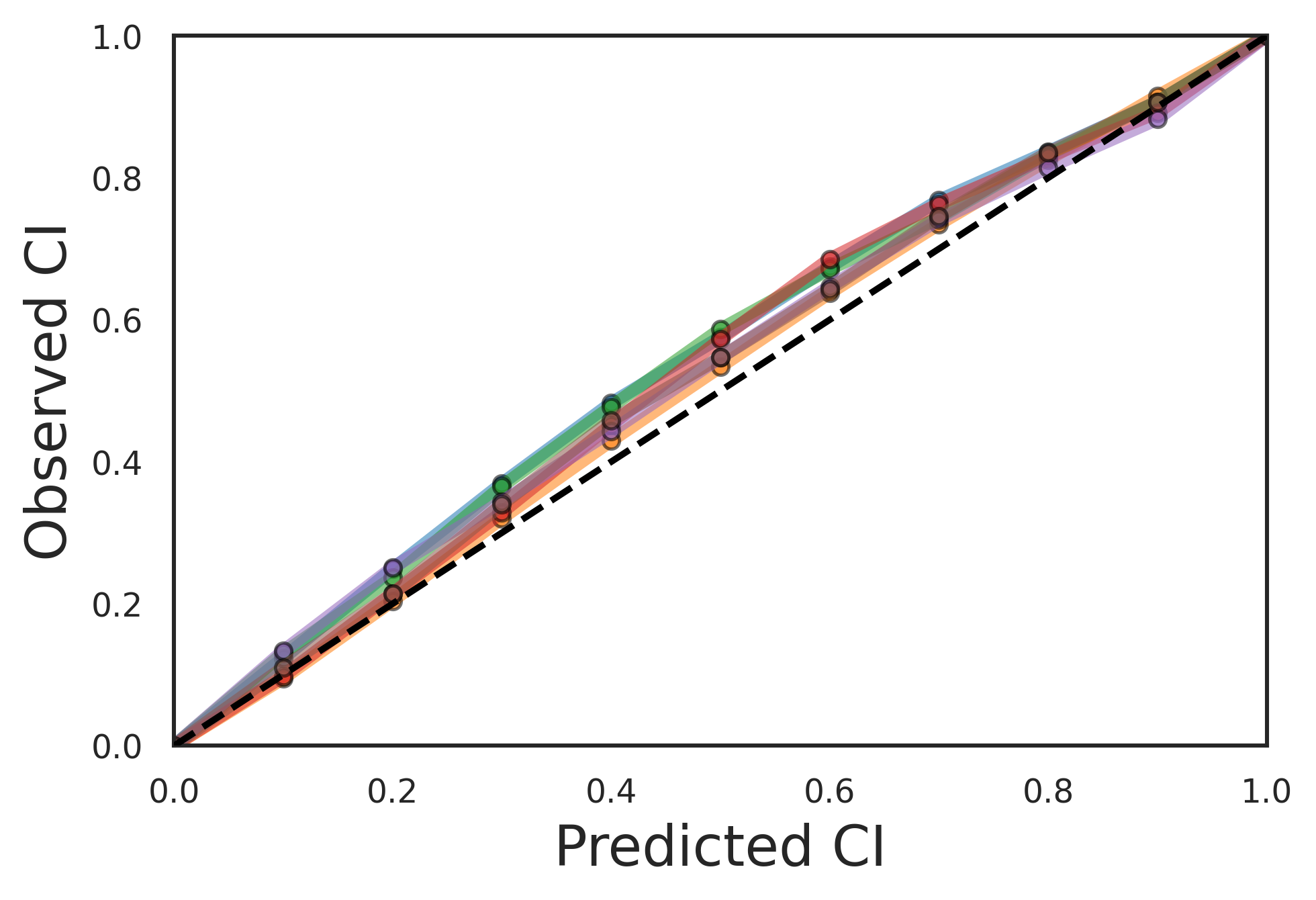}
        \hspace{0.0\columnwidth}
        \includegraphics[height=0.18\textheight]{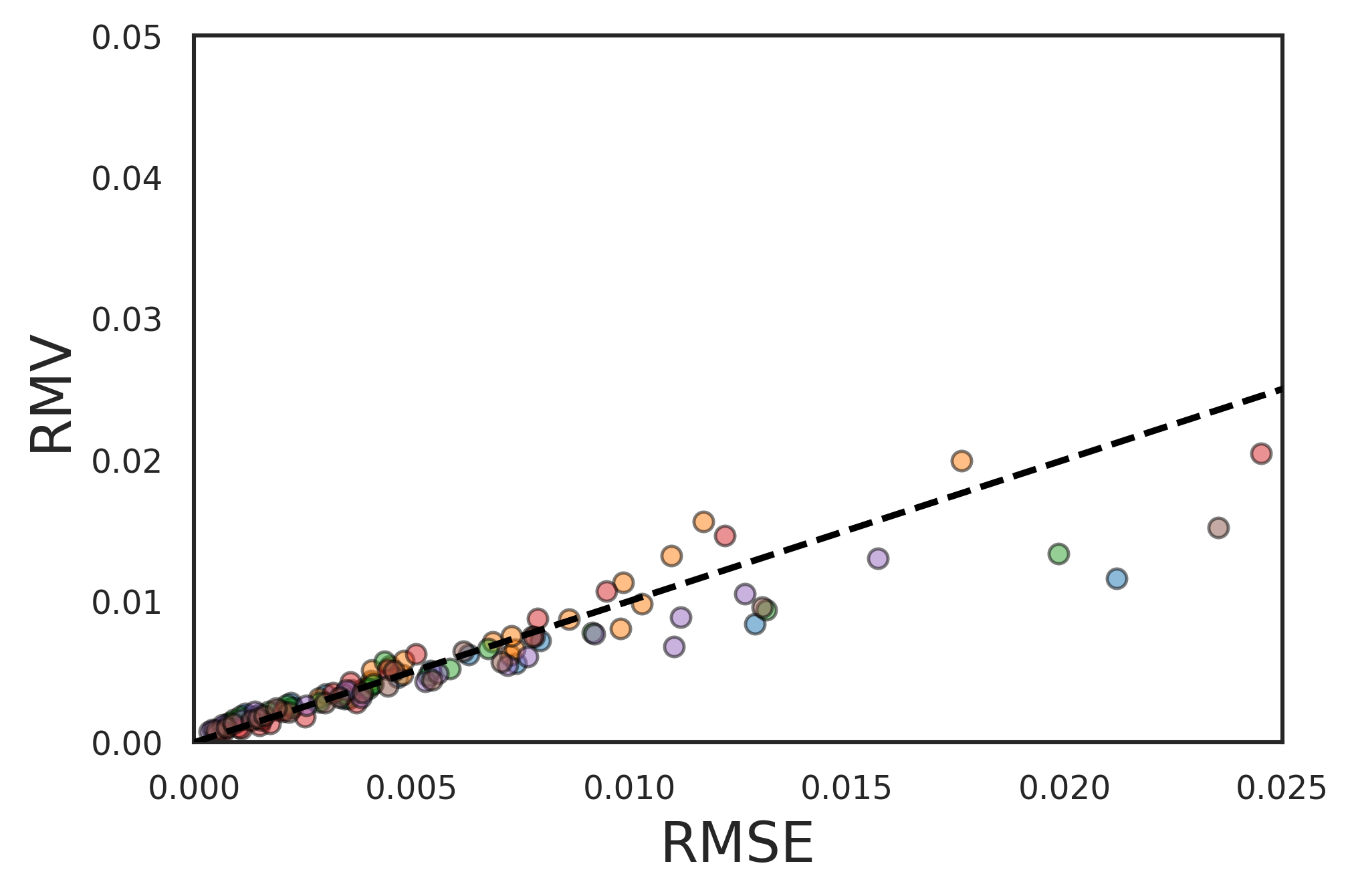}
        \caption{DE-16}\label{fig:calcomp_NN_d}
    \end{subfigure}
    \caption{Reliability plots of DE models after STD calibration: (left) CI-based reliability plots, (right) error-based reliability plots.}
    \label{fig:reliab_DE_cal}
\end{figure*}

\bibliographystyle{elsarticle-harv} 
\bibliography{cas-refs}





% \end{thebibliography}
\end{document}